\documentclass[a4paper,journal]{IEEEtran}
\IEEEoverridecommandlockouts

\usepackage[T1]{fontenc}

\usepackage[table,xcdraw]{xcolor}
\usepackage{tikz}
\usepackage{cite}
\usepackage{amsmath,amssymb,amsfonts,amsthm}
\usepackage{dsfont}
\usepackage{algorithm}
\usepackage{algorithmic}
\usepackage[graphicx]{realboxes}
\usepackage{textcomp}
\usepackage{xcolor}
\usepackage{multirow}
\usepackage{booktabs}
\usepackage{paralist}
\usepackage{url}
\usepackage{nicefrac}

\usepackage{changes}
\definechangesauthor[name={TODO}, color=green]{TODO}
\definechangesauthor[name={Giosuè Marinò}, color=cyan]{GM}
\definechangesauthor[name={Marco Frasca}, color=blue]{MF}
\definechangesauthor[name={Alessandro Petrini}, color=red]{AP}
\definechangesauthor[name={Dario Malchiodi}, color=magenta]{DM}

\def\BibTeX{{\rm B\kern-.05em{\sc i\kern-.025em b}\kern-.08em
    T\kern-.1667em\lower.7ex\hbox{E}\kern-.125emX}}
\usepackage{fancybox}
\usepackage{fixmath}
\newcommand{\inda}{\phantom{1}\hspace{0.5mm}}
\newcommand{\indb}{\phantom{1}\hspace{3.5mm}}

\newcommand{\bW}{\boldsymbol{W}}
\newcommand{\bWo}{\boldsymbol{W}^\mathrm{o}}
\newcommand{\HAC}{{\tt{HAC}}} 
\newcommand{\sHAC}{{\tt{sHAC}}} 
\newcommand{\bx}{\boldsymbol{x}}
\newcommand{\bB}{\boldsymbol{B}}

\newcommand{\bv}{\boldsymbol{v}}
\newcommand{\br}{\boldsymbol{r}}
\newcommand{\bc}{\boldsymbol{c}}
\newcommand{\bnz}{\boldsymbol{nz}}
\newcommand{\bz}{\boldsymbol{z}}
\newcommand{\bcb}{\boldsymbol{cb}}
\newcommand{\bri}{\boldsymbol{ri}}

\newcommand{\bA}{\boldsymbol{A}}
\newcommand{\bI}{\mathbold{\Pi}} 
\newcommand{\bH}{\boldsymbol{H}}
\newcommand{\bU}{\boldsymbol{U}}
\newcommand{\bSigma}{\boldsymbol{\Sigma}}
\newcommand{\bV}{\boldsymbol{V}}
\newcommand{\bX}{\boldsymbol{X}}

\newtheorem{theorem}{Fact}
\newtheorem{corol}{Corollary}
\newtheorem{example}{Example}

\DeclareMathAlphabet{\mathpzc}{OT1}{pzc}{m}{it}

\usepackage{xr}

\makeatletter
\newcommand*{\addFileDependency}[1]{
  \typeout{(#1)}
  \@addtofilelist{#1}
  \IfFileExists{#1}{}{\typeout{No file #1.}}
}
\makeatother
\newcommand*{\myexternaldocument}[1]{%
    \externaldocument[S-]{#1}%
    \addFileDependency{#1.tex}%
    \addFileDependency{#1.aux}%
}
\myexternaldocument{SupplementaryMaterial}

\makeatletter
\newcommand{\linebreakand}{%
  \end{@IEEEauthorhalign}
  \hfill\mbox{}\par
  \mbox{}\hfill\begin{@IEEEauthorhalign}
}
\makeatother

\begin{document}

\title{Compact representations of convolutional neural networks via weight pruning and quantization}

\author{\IEEEauthorblockN{Giosuè Cataldo Marinò, Alessandro Petrini, Dario Malchiodi, Marco Frasca}\\
\IEEEauthorblockA{\textit{Università degli Studi di Milano},
Milano, Italy}\protect\\
E-mail: giosue.marino@studenti.unimi.it, \{alessandro.petrini, dario.malchiodi, marco.frasca\}@unimi.it}

\markboth{Journal of \LaTeX\ Class Files,~Vol.~14, No.~8, August~2015}%
{Shell \MakeLowercase{\textit{et al.}}: Bare Demo of IEEEtran.cls for IEEE Journals}

\maketitle

\begin{abstract}
The state-of-the-art performance for several real-world problems is currently reached by convolutional neural networks (CNN). Such learning models exploit recent results in the field of deep learning, typically leading to highly performing, yet very large neural networks with (at least) millions of parameters. As a result, the deployment of such models is not possible when only small amounts of RAM are available, or in general within resource-limited platforms, and strategies to
 compress CNNs became thus of paramount importance.
In this paper we propose a novel lossless storage format for CNNs based on source coding and leveraging both weight pruning and quantization.
We theoretically derive the space upper bounds for the proposed structures, showing their relationship with both sparsity and quantization levels of the weight matrices.
Both compression rates and execution times have been tested against reference methods for matrix compression, and an empirical evaluation of state-of-the-art quantization schemes based on weight sharing is also discussed, to assess their impact on the performance when applied to both convolutional and fully connected layers. {On four benchmarks for classification and regression problems and comparing to the baseline pre-trained uncompressed network, we achieved a reduction of space occupancy up to $0.6\%$ on fully connected layers and $5.44\%$ on the whole network, while performing at least as competitive as the baseline}.
\end{abstract}


\IEEEpeerreviewmaketitle
\urlstyle{tt}

\section{Introduction}


The methodology behind deep neural networks (DNNs) dates back to more than forty years ago.
However, the availability of dedicated hardware (such as GPUs or TPUs) and of huge datasets recently allowed to maximize the performance of several DNN-based predictors, setting in practice the state-of-the-art for several problems of image processing, financial forecasting, and so on.
Convolutional neural networks (CNNs) played a key role in this advancement, and several such pre-trained models,
such as for instance AlexNet~\cite{krizhevsky2012imagenet} and VGG16~\cite{Simonyan15},
are avaliable for use
as base models for the applications of transfer learning techniques~\cite{transfer-learning}.
In any case, such models have a considerable memory footprint: for instance, the above mentioned VGG16 demands around 500 MB.
This also impacts on the energy consumption required to query such models, thus limiting their use in mobile phones, smartwatches, and in general within the IoT world. As a matter of fact, the need of \emph{space-conscious} models is actually emerging in several machine learning applications~\cite{Ferragina:2020pgm}.
Although some learning approaches directly produce succint models~\cite{sandler2018mobilenetv2},
in this paper we consider the problem of \emph{compressing} pre-trained CNNs, whilst not altering their structure, so as to be able to reuse the wealth of available models and their full capabilities. This is not a limitation, as other compression approaches varying the network topology can be earlier applied.
Specifically, we aim at preserving the original structure of pre-trained models while suitably adjusting the representation of their parameters, as well as the way they are stored. To this end, two novel storage formats 
are presented here, namely \emph{Huffman Address Map} (\HAC) and \emph{sparse Huffman Address Map} (\sHAC), able to extract the knowledge distributed onto millions, or even billions, of connection weights of a learnt network, transforming it in another structure exhibiting a considerably smaller memory footprint, without sacrificing its performance (or even improving it).
This is achieved by jointly applying lossy compression schemes for the NN weights and entropy coding techniques for the lossless representation of the compressed weights. Our goal to not modify the original network structure imposed to select in the vast domain of CNN compression methodologies those ensuring this property, that is weight pruning and quantization. Subsequently, the proposed representations are devised to benefit from both these two techniques,
and combine entropy coding, address maps, and compressed sparse column (CSC) representations.
Concerning the compression schemes, the top-performing quantization methods in literature have been considered, including a recently proposed probabilistic technique borrowed from the realm of federated learning.

We tested the proposed methodology using two publicly available CNNs and four benckmarks (two datasets for image classification and two for the prediction of drug-target affinity), obtaining as a result the confirmation that a well-conceived compression can even lead to better performance w.r.t.\ uncompressed models.
These experiments provided important indications about the behaviour of individual compression schemes when applied to both convolutional and dense layers,
with the latter yielding in average higher performance and compression rates. When applied to the whole model (hence to both types of layer simultaneously), our representations achieved compression rates up to around $20\times$, {while not worsening the model performance.}

The paper is organized as follows: Sect.~\ref{sec:related-works} describes the principal approaches for CNN compression, while Sects.~\ref{sec:compression-techniques} and \ref{sec:compressed-representations} describe the considered compression and representation techniques. Section~\ref{sec:experiments} illustrates the above mentioned experimental comparison, depicted in terms of performance gain/degradation, achieved compression rate, and execution times. Some concluding remarks end the paper.

\section{{Related work}}
\label{sec:related-works}
Existing DNN compression methods can be classified into five broad categories, i.e., weight pruning, weight quantization, low-rank matrix and tensor decomposition, knowledge distillation and structural compression. In this distinction, the techniques based on weight sharing are included in the broader category of weight quantization. {The above mentioned categories are separately described here below. The reader interested to a thorough review of compression methods for NNs can refer for instance to~\cite{deng2020model,cheng2017survey}.}
{\subsection{Weight pruning}}
\label{sec:weight_pruning}
Weight pruning is likely the most commonly used technique to lower the number of parameters in a DNN.
It consists in eliminating 
connections deemed as irrelevant,
thus reducing the number of parameters, while clamping the weight  matrix dimension ({unlike in structural compression, see Sect.~\ref{sub:structCompr}}).
Libraries for sparse matrix multiplication (SM) can be used
to take full advantage of
the memory reduction.
However, SM tends to be slower than its dense counterpart;
as a consequence, structural compression is often jointly applied with pruning. The simplest pruning strategy involves
to set a threshold $\tau$ and remove weights whose absolute value is lower than $\tau$~\cite{Hagiwara93}. The threshold can be layer specific, or
it can be set for the whole network.
This criterion, often
called magnitude-based pruning,
involves a subsequent fine-tuning (retraining) of the remaining weights. In another class of pruning approaches, ($L_1$  or $L_2$) regularization terms
drive the learning algorithm to output a network in which several weights have negligible values, so that pruning
becomes
a straightforward post-training step~\cite{Weigend90}.

Pruning has also been performed via
genetic algorithms~\cite{WHITLEY90} and particle swarm optimization~\cite{Juanjuan10}.
Using
these techniques, weights
initially removed
can be
reinserted;
however, their complexity allows application to DNNs of limited size.

{\subsection{Weight quantization}\label{sub:WQ}
}
{In neural networks, \emph{quantization}
consists in
fixing the number of bits used
to represent weights, activations or gradient values. This leads to a compressed model for instance when single precision floating point (FP32, as each value requires 32 bits) is used in place of standard double precision.
}
Half-precision
FP
(FP16) and integer arithmetic (INT16) are also commonly considered. To reduce memory 
{footprint, and}
speed up training,
even smaller precisions have been recently evaluated,
considering short integers (INT8, INT4, INT2), or just 1-bit~\cite{Dally15}.
{In any case, using less than 32 bits}
can yield unacceptable performance decays. A standard approach in this context is to replace
FP operations with integer ones, sometimes adding a further training ameliorating the accuracy compromise during this replacement~\cite{Jacob18}. Linear and logarithmic schemes
using
less than 8 bits  have been applied 
achieving a limited accuracy drop~\cite{Hubara17}. Stochastic 
techniques
also have a role in
{letting FP16-based DNNs}
converge
to
a test accuracy 
{comparable to that of}
their FP32
{counterparts~\cite{Gupta15}.}
{An extreme quantization
takes place with \emph{binarization}, in which}
a single bit and logic gates are used for representation and FP operations,
with a strong
memory reduction
at the expense of
an accuracy
drop~\cite{Hubara16}.
{A}
hybrid approach
quantizes values in
{narrow}
regions,
using higher precision for the remaining values.
Thus,
a smaller number of bits 
is used for the regions in which most values lie, leading
to smaller quantization errors~\cite{Park18}.
Loss-aware
{strategies}
have also been used to contain the penalty induced by quantization. In~\cite{Hou17} the compression scheme $\bW=\alpha\bB$ is adopted,
where $\alpha \in \mathbb{R}$ and $\bB$
have binary entries,
and the loss is minimized over $\alpha$ and $\bB$.
Similarly, weights are first divided in two groups, then only one group is quantized
using solely powers of $2$ as weights~\cite{Zhou17}.
In weight sharing techniques, weights are partitioned into $k$ categories,
and a representative value replaces all weights in its category.
When
$k$ is low,
representatives can be stored
with full precision in a vector $\br$, and keep in memory the matrix $\bI$ containing pointers to $\br$; as pointers require less space than FP32 weights (e.g., INT8 when $k\leq 256$), $\bI$ is more compact than $\bW$ and largely compensates for the additional $\br$~\cite{Han15}.
We will refer to this approach as \textit{index map}. These methods mainly differ in
the representative selection,
e.g., by means of clustering~\cite{Han15}, statistical methods~\cite{HAMICPR,RRPR21},
{and partitioning schemes~\cite{Choi20,ECSQ91}.
}
Details about these methods are given in Sect.~\ref{sec:compression-techniques}.

{\subsection{Low-rank  matrix  and  tensor  decomposition}\label{sub:LRD}}
DNNs can be compressed by decomposing weight tensors 
in a lower rank approximation:
a matrix $\bW \in \mathbb{R}^{n\times m}$ of full rank $r$ can be decomposed as $\bW=\bA\bH$, where $\bA \in \mathbb{R}^{n\times r}$ and $\bH \in \mathbb{R}^{r\times m}$, moving space complexity from $\mathcal{O}(nm)$ to $\mathcal{O}(r(n+m))$, with some approximation error coming from the estimate of low-rank matrices.
Singular Value Decomposition (SVD) has been widely used to achieve such a decomposition~\cite{Xue13,Sainath13}:
$\bW$ is factorized as
$\bW= \bU\bSigma\bV^T$, where $\bU \in \mathbb{R}^{n\times r}$ and $\bV \in \mathbb{R}^{m\times r}$ are orthogonal matrices, and $\bSigma \in \mathbb{R}^{r\times r}$ is the diagonal matrix of singular values. The nonzero elements of $\bSigma$ are sorted in decreasing order (along with rows of $\bU$ and $\bV$), and the top $q$ rows are used in the approximation $\bW\approx \bU_q\bSigma_q{\bV_q}^T$, where
$\bX_q$ is the sub-matrix containing the first $q$ rows of $\bX$. If $q < r$, it is called truncated SVD. This approach has been branched out also to tensors in recurrent NNs~\cite{DeLathauwer08}
and in CNNs~\cite{Rigamonti13,Jaderberg14,Yu17}.

{\subsection{Knowledge distillation }\label{sub:KD}}
This compression technique encompasses different approaches to learn a `thinner' DNN model, called \textit{student}, from a larger \textit{teacher} model,
{whose outputs act as soft targets for the training process.}
{The teacher output should be a probability distribution, and the idea is to exploit the corresponding logits
to `distill' information to the student, which is trained minimizing the cross entropy between the teachers' logits and its logits~\cite{Caruana14}.}
{More precisely, the softmax layer transforms a logit $z_{j}$
w.r.t.\ the $j$-th class, producing the probability $q_{j}$
to be associated with that class as follows:
\begin{equation*}
    q_{j} := \phi(z_{j}):=\frac{\mathrm e^{z_{j}/T}}{\sum_k \mathrm e^{z_{k}/T}},
\end{equation*}}
where $T>1$ is a temperature parameter.
{Now, let $v_{j}$ and $z_{j}$ denote the logits of teacher and student, respectively, and $q_{j} = \phi(z_{j})$, $p_{j} = \phi(v_{j})$.
}
In~\cite{hinton15}, knowledge distillation (KD) is shown to be a special case of matching logits among two models, leading to the following {cross entropy} loss:
\begin{equation}\label{eq:KDopt}
{E(\bz, \bv) = -\sum_j p_j\log q_j.}
\end{equation}
The derivative of $E$ with respect to $z_{j}$ is then calculated as
\begin{equation*}
{\frac{\partial E}{\partial z_{j}} = \frac{1}{T} (q_j- p_j) = \frac{1}{T}\Big(\frac{\mathrm e^{z_{j}/T}}{\sum_k \mathrm e^{z_{k}/T}} - \frac {\mathrm e^{v_{j}/T}}{\sum_k \mathrm e^{v_{k}/T}}
\Big)},
\end{equation*}
which allows to minimize  Eq.~(\ref{eq:KDopt}) and to improve the student performance exploiting the pre-trained teacher's (or ensemble of teachers) generalization ability.
When distilling a linear classifier,
the student is able to learn exactly the teacher's outcome if the number of examples is
{greater or equal to the size of the original training set.}
{
When less data is available, the student finds the best approximation of the teacher’s weight vector within}
the subspace spanned by training data~\cite{phuong19}.

A limitation emphasized for 
KD is that, although the student  is able to reproduce
the behavior of the teacher on training data,
{it might loose its generalization ability.}
Early stopping in training the student can mitigate this drawback.
Moreover, it has been discussed that the student size cannot be excessively small:
{a big gap between teacher and student}
leads to improper knowledge transmission~\cite{TA19}.
The \emph{teacher assistant} (TA) model,
distilling
{an intermediate TA network from the teacher, permits to mitigate such difficulty,}
in the so-called \emph{multi-step} KD.
Further, label smoothing has been shown  leading to performance degradation since it provides less information to the student about class boundaries~\cite{Muller19}. 
Finally, a contrastive loss between teacher and student outputs has also been proposed
{in place of cross entropy,}
to preserve structural information of the embedding space~\cite{Tian20}.\\

{\subsection{Structural Pruning}\label{sub:structCompr}}
In structural compression,
inadequate components {(e.g., units or layers)} are pruned, usually through iterative procedures~\cite{MallyaL18}.
The resulting
NN is
faster,
low power consuming, and memory-efficient, at the expense of a
contained accuracy drop.
Adequacy is related to the
loss change incurred when
a
component is
removed ~\cite{Molchanov19}.
\emph{Skeletonization}, for instance,
{consists in coupling each connection to an importance coefficient.
The more general class of \emph{loss sensitivity} methods calculates importance coefficients via measures of loss variation using first~\cite{Molchanov16} or second order derivatives~\cite{Lecun90}, and after a joint training of connections and coefficients, all units having the lowest importance coefficients are eliminated.

{Dense} network layers have been for instance shrunk by cutting away neurons isolated during an iterated weight pruning,
adding
a smaller number of new neurons to improve performance,  so that the overall number of parameters monotonically decreases~\cite{Han13}. {Analogously, filters of convolutional layers are removed in~\cite{Li17} when their contribution to the overall accuracy is negligible.}
{As the number of channels in a given layer does not change across filters,}
the sum of the weight magnitudes can be computed as an average of weight value for each {filter, pruning the latter if this average is small.}
An efficient approach, named ThiNet, includes in the learning objective the rate of filters to be retained based on the output to the next layer:
{when this output}
can be accurately approximated using only a subset of its input channels, the other filters 
can be eliminated~\cite{Thinet}. 
\emph{Channel pruning} has also been proposed to lessen computation and storage requirements,
{removing unimportant channels
according to their relevance in determining the layer output~\cite{He17}.}
Finally, \emph{layer pruning} 
removes
some selected layers, mainly when striking compression is required (e.g., for deploying on smart devices)~\cite{Nips18,Tan19,Chauhan18}. However,
this
results in higher accuracy {decay}
due to structural deterioration of the DNN model, as some layers have specific semantics.
\bigskip

\section{Compression techniques}
\label{sec:compression-techniques}
{As mentioned above, in this work we focus on compressions which do not alter the structure of the pre-trained network. Therefore, the most promising state-of-the-art weight pruning and quantization techniques have been considered, with the aim to store their result via the structures proposed in Section~\ref{sec:compressed-representations}  (designed to leverage their compression properties). The description of such compression strategies is supplied in the following, along with some preliminary definitions.}
\subsection{{Preliminary definitions}}\label{sub:prel}
The connection weights of one layer in the network are denoted by a matrix $\bWo \in \mathbb R^{n \times m}$, whereas its compressed version is $\bW$, having the same dimension of $\bWo$.
Symbols $w^\mathrm{o}$ and $w$ denote generic entries of $\bWo$ and $\bW$, respectively. The \emph{occupancy ratio} of $\bW$ is defined as $\psi = \frac{\mathrm{size}(\bWo)}{\mathrm{size}(\bW)}$ (reciprocal of \textit{compression ratio}), where $\mathrm{size}(x)$ is the memory size of $x$.
Boldface and italic boldface is used for matrices and vectors (e.g., $\bW$ and $\bx$), while $|\cdot|$ is a cardinality operator returning the length of a string or the number of elements in a vector; $\mathds{1}_A$ denotes the characteristic function of a set $A$. The $\log$ function always refers to the binary logarithm. Finally, {$s\in [0,1]$ is the ratio of non-zero elements in $\bW$ (number of non-zero entries divided by $nm$), and $1-s$ its \emph{sparsity coefficient}}.

\subsection{Weight Pruning}
{We implemented weight pruning (see Sect.~\ref{sec:weight_pruning}) removing weights
that are
small in absolute value.
After having fixed an empirical percentile $w_p$ of the entries of $\bWo$, we defined $\bW$ by setting $w = w^\mathrm{o}$ if $|w| > w_p$, $0$ otherwise.
The time complexity is $\mathcal O(nm \log (nm))$, as the sorting step needed for the computation of $w_p$ dominates the overall procedure. We subsequently retrained the network on the same dataset, only updating non-null weights in $\bW$.}
The only hyper-parameter is the percentile level $p$,
related in turn
to the sparsity coefficient.
(see Sect.~\ref{sec:experiments} for a description of how {hyper-parameters of all considered compression methods have been selected).}

\subsection{{Quantization via weight sharing}}\label{sub:WSquant}

{
This quantization strategy
consists in reducing the space needed to store individual weights via weight sharing (WS), expressly by casting connection weights into categories and substituting all weights in a category with their representative. This allows quantization by replacing weights with their category index. In the following, the state-of-the-art approaches in this context considered for this work are are briefly introduced.}
\subsubsection{{Clustering-based WS (CWS)}}\label{subsub:CWS}
this strategy, fixed the number $k$ of clusters,
aims at gathering similar values in $\bW^\mathrm{o}$ via the $k$-means algorithm~\cite{mcqueen},
{obtaining the corresponding centroids $\{c_1, \dots, c_k \}$, and subsequently replacing each weight in $\bWo$ with the corresponding centroid.
}
Centroids
are
stored in a vector $\bc$
whose indices
populate
the index map $\bI$~\cite{Han15}.
Thus, if
$w^\mathrm{o}_{ij}$ is associated with centroid, say, $c_1$, then $\pi_{ij} = 1$.
Denoted by $b$ and $\bar b$ the number of bits respectively used to store one entry of $\bWo$ and $\bI$, the
occupancy ratio
is given by $\frac{\bar b nm + kb}{bnm}=\frac{\bar b}{b} + \frac{k}{nm}$. For instance, when $k\leq 256$, $\bar b=8$, and assuming FP32
for $\bW^\mathrm{o}$ ($b=32$), the occupancy would be
$\simeq \nicefrac{1}{4}$.
This comes
at the price of 
two memory accesses in order to retrieve a weight.
The time complexity
is $\mathcal O(k(mn)^2)$ (due to $k$-means).
A retraining phase is then applied, ensuring weights always assume values in the centroid set. This is achieved by using the cumulative gradient
\begin{equation*}
    \frac{\partial \mathcal L}{\partial c_l} = \sum_{i,j} \frac{\partial \mathcal L}{\partial w_{ij}}\mathds{1} (\pi_{ij}=l),
\end{equation*}
where $l \in \{ 1, \dots, k \}$.
This
might end up in using less than $k$ weights, if two or more centroids converge to a same value during retraining.
To achieve  a higher compression, in~\cite{Han15} pruning and CWS have been applied in chain, with weight sharing considering non-null weights identified by pruning.

\subsubsection{Probabilistic WS (PWS)}
this technique
is based on a weight sharing technique named \textit{Probabilistic Quantization}, recently proposed in~\cite{HAMICPR} and
relying
{on a probabilistic transformation analogous to those mapping weights onto special binary or ternary values proposed in~\cite{NIPS2015_5647,deng2018gxnor}. Given}
$\underline w = \min \bWo,\  \overline w = \max \bWo$,
PWS is based on the following probabilistic rationale: suppose that each learnt weight $w^\mathrm{o}$ is the specification of a random variable $W^\mathrm{o}$ 
with support $\mathcal W := [\underline w, \overline w]$. If $W$ denotes the two-valued random variable defined by
$
    \mathrm P(W = \underline w) = \frac{\overline w - w}{\overline w - \underline w}, \quad
    \mathrm P(W = \overline w) = \frac{w - \underline w}{\overline w - \underline w}\enspace$,
the specifications of $W$ approximate a weight $w$ through an extreme form of weight sharing, now only using two representative values. It is easy to show that $\mathcal E(W | W^\mathrm{o}=w) = w$, so that independently of the distribution of $W^\mathrm{o}$,
\begin{multline*}
\mathcal{E}(W) =
\int_\mathcal W \mathcal{E}(W | W^\mathrm{o}=w) f_{W^\mathrm{o}}(w) \mathrm d w =
\\ \int_\mathcal W w f_{W^\mathrm{o}}(w) \mathrm d w = \mathcal E(W^\mathrm{o}).
\end{multline*}

{As a consequence, pseudorandomly extracting a specification of $W^\mathrm{o}$ for each entry $w^\mathrm{o}$ and building new matrix $\bW$ using these specifications as entries, we obtain an highly compressible \emph{unbiased} estimator of $\bWo$.}
The method is extendable by partitioning $\mathcal W$ in $k > 2$ intervals. A generic $w^\mathrm{o}$ would be obtained
{exactly as above, by referring to $\overline w$ and $\underline w$ as the extremes of the sub-interval containing $w^\mathrm{o}$.}
{The $k$ sub-intervals should be chosen preserving unbiasedness:}
some knowledge about the distribution of $\bWo$ would help, like in the following example.

\begin{example}
Assume each element of $\bWo$ to be distributed uniformly over $\mathcal W$, and partition the latter set in $k$ sub-intervals evenly spaced
$[p_{i*}, p_i^*)$, for $i = 1, \dots, k$, where
\begin{eqnarray*}
p_{i*}  =  \underline w + \frac{i-1}{k} (\overline w - \underline w), \hspace{0.5cm}
p_i^*  =  \underline w + \frac{i}{k} (\overline w - \underline w).
\end{eqnarray*}
Uniformity implies $f_{\bWo}(w)= \frac{1}{\overline w - \underline w}\mathds{1}_{\mathcal{W}}(w)$, so that
\begin{displaymath}
    \mathrm{P}(\bWo \in P_i) = \int_{\underline w}^{\overline w} f_{\bWo}(w)\mathrm{dw} = \frac{p_i^*-p_{i*}}{p^*-p_*} = \frac{1}{k} \enspace,
\end{displaymath}
having all sub-intervals the same length. Similarly, it can be shown that the variance of the biased estimate $\bW$ quadratically decreases with $k$ and linearly increases w.r.t.\ $|\mathcal W|$.
\end{example}

In general,
unbiasedness can be ensured by fixing the intervals' extremes as 
$\chi_{\frac{i}{k}}$, for $i = 0, \dots, k$, where $\chi_q$ denotes the $q$-quantile of $\bWo$.
{The time and space complexity respectively amount to $\mathcal O(nm \log (nm))$ (due to quantile computation) and $\mathcal O(mn + k)$.}
{Note that, analogously to CWS, post-compression retraining using the cumulative gradient is needed for PWS, which can also be applied in combination with pruning.}

\subsubsection{Uniform Quantization (UQ)}
{this quantization scheme, which selects representative weights uniformly  in the weight domain, has been proven yielding an entropy  asymptotically smaller than that of any other quantizer, regardless of the source statistics, under the assumption that the source has a reasonably smoothed density function~\cite{GishH1968}.
Banking upon this result, in~\cite{Choi20} UQ has been formalized so as to transform the weight $w^\mathrm{o}$ according to the following scheme:
\begin{equation*}
w = \delta \cdot \mathrm{round}((w^\mathrm{o} + d)/\delta)) - d\ ,
\end{equation*}
where $\delta > 0$ is the interval size, $d\in [-\frac{\delta}{2}, \frac{\delta}{2}]$ is a constant bias, and $\mathrm{round}$ is the rounding function.
$\delta$ must be selected based on compression ratio and/or accuracy requirements and the number $k$ of distinct weights desired: compression ratio increases while accuracy degrades and $k$ decreases, as $\delta$ grows.}
\subsubsection{Entropy Constrained Scalar Quantization (ECSQ)}
this compression strategy (also known as Entropy Coded Scalar Quantization)~\cite{ECSQ91,chou1989entropy} transforms a signal described through a random variable onto quantized values. In the realm of NN, the random variable, its specifications, and the quantized values correspond to
$W^\text{o}$, $w^\text{o}$, and $w$ in our notation.
ECSQ partitions the domain of $W^\text{o}$ in $k$ \emph{decision levels}, each identified by an interval $\mathcal W_i$, and selects an analogous number of \emph{representation levels} $w_i$. 
All values belonging to a decision level are transformed onto the corresponding representation.
Decision and representation levels are chosen jointly optimizing the expected value for the quantization distortion $D$ (using a prefixed distortion measure such as MSE) and the entropy $H$ of the resulting distribution of representation levels. The optimal ECSQ scheme has been found minimizing distorsion constraining entropy to not exceed a prefixed threshold~\cite{GishH1968}, considering for instance the optimization of the Lagrange cost
\begin{displaymath}
    D + \lambda H =
    \frac{1}{nm}\sum_{i=1}^{k}
      \sum_{j\in \mathcal W_i}
        \left(
          \left| w^\text{o}_j - w_i \right|^2
          - \lambda \log_2 p_i
        \right)\ ,
\end{displaymath}
where $\lambda$ is a Lagrange multiplier, and $p_i = |\mathcal W_i|/(nm)$.
ECSQ and UQ have shown the best trade-off accuracy/compression rate in a recent state-of-the-art comparison~\cite{Choi20}.}
\section{Compressed Matrix Representation}
\label{sec:compressed-representations}
Once the techniques described in Sect.~\ref{sec:compression-techniques} have been applied, the resulting matrix $\bW$ has the same dimensions of the original matrix.
Here we first describe the classical CSC format, then the proposed compact structures to store $\bW$, exploiting the sparsity and presence of repeated values.
\subsection{Compressed sparse column}\label{sub:CSC}
  The \emph{compressed sparse column} (CSC) format \cite{CSC} is a
  standard
  for storing sparse matrices. It is composed of $3$ arrays:
\begin{itemize}
 \item[-] $\bnz$, containing the nonzero values, listed by columns;
 \item[-] $\bri$, containing the row indices of elements in $\bnz$;
 \item[-] $\bcb$, where the difference $cb_{i+1}-cb_i$ provides the number of nonzero elements in column $i$; thus, $\bcb$ has dimension $m+1$, where $cb_{m+1}=cb_1 + |\bnz|$.
\end{itemize}
\begin{example}\label{ex:CSC}
Consider the matrix

\begin{displaymath}\label{eq:mat_ex}
\bW =
\begin{pmatrix}
1 & 0 & 4 & 0 & 0\\
0 & 10 & 0 & 0 & 0\\
2 & 3 & 0 & 0 & 5\\
0 & 0 & 0 & 0 & 0\\
0 & 0 & 0 & 0 & 6\\
\end{pmatrix}.
\end{displaymath}
Its CSC representation is
$\bnz = (1, 2, 10, 3, 4, 5, 6)$, $\bri = (1, 3, 2, 3, 1, 3, 5)$, and $\bcb = (1, 3, 5, 6, 6, 8)$.
\end{example}
Let $q = |\bnz|$ be the number of nonzero elements in $\bW$,
{and denote henceforth by $b$}
the number of bits used to represent every element of the matrix (one memory word), so that we need $bnm$ bits to store $\bW$, and $(2q+m+1)b$ to store its CSC representation{\footnote{we assumed $b$ bits also for the components of $\bri$, although they can be represented using only ${\lceil\log n\rceil}$ bits, which might be lower than $b$.}}.
The occupancy ratio is given by $\psi_{CSC} = \frac{2q+m+1}{nm}$.
Given the non-zero ratio $s$ of $\bW$, it holds $q=snm$, thus $\psi_{CSC}<1$ implies $s < \frac{1}{2} - (\frac{m+1}{2nm})$.
The  dot product $\bx^T\bW$
{can be computed using a custom procedure for the CSC format, having a time complexity of~$\mathcal{O}(q)$ \cite{CSC} and which can be sped up through parallel computing.}
{Using $b$ bits for each matrix element constitute the main limitation of CSC, whereas better compactness and bit-memory efficiency can be achieved using variable-lenght coding.}
 \subsection{Huffman address map compression}\label{sub:ham}
 {The idea of using Huffman coding for the representation of the weight matrix resulting from pruning and quantization has been suggested, but not realized, in~\cite{Han15}. In this section we provide our realization and its thorough explanation, along with its space upper bound first derived here. 
 }
 {We call the resulting technique \textit{Huffman Address Map compression} (\HAC), as it is based on Huffman coding and address map logic; we remark that this technique applies a lossless compression, as well as in CSC.}
{Address maps organize the matrix entries as a row- or column-order based sequence of bits, in which $0$ identifies any null entry, while each remaining element $z$ is represented by a binary string $a(z)$ encoding its address.
\begin{example}
The bit sequence corresponding to the column-order based address map for the matrix $\bW$ of Example~\ref{ex:CSC} is
\begin{displaymath}
    a(1) 0 a(2) 0 0 0 a(10)a(3)00a(4)00000000000a(5)0a(6)\ .
\end{displaymath}
\end{example}
}
{Of course, in order to achieve efficiency it is necessary to rely on compact representations for addresses. We implemented $a(z)$ via the corresponding Huffman coding $H_{\bW}(z)$, in view of its well-known properties: it is instantaneous, uniquely decodable and it has a near-optimal compression rate~\cite{Huffman52}. More precisely, if we consider a source $(w_1, \ldots, w_l)$, the corresponding probabilities $(p_1, \ldots, p_l)$, and denote
\begin{itemize}
    \item by $\mathcal{H}=-\sum_{i=1}^l p_i \log p_i$ the source entropy, which
    by Shannon's source coding theorem
    corresponds
    to the minimal average number of bits per symbol~\cite{Shannon48}, and
    \item by $\overline H_{\bW} := \sum_{i=1}^l p_i |H_{\bW}(w_i)|$ the average number of bits per symbol attained when using the Huffman coding,
\end{itemize}
it can be shown that  $\mathcal{H} \leq |\overline H_{\bW}| \leq \mathcal{H} + 1$.}


To get uniquely decodable strings
{we also include zeroes in the Huffman code.
This brings us to a total of $q + 1$ codewords. We denote by $\HAC(\bW)$ the bit stream encoding $\bW$, and split the former into $N=\lceil {|\HAC(\bW)|}/{b}\rceil$ memory words, in turn denoted as $\HAC(\bW)_1, \ldots, \HAC(\bW)_N$ and represented as an array $\mathcal{C}_{\HAC}(\bW)$ of $N$ unsigned integers. Zero-padding is added to
the last word when $|\HAC(\bW)|$ is not a multiple of $b$.}
{
\begin{theorem}(\HAC\ worst case)\label{fact1}
{If $\bW$ is dense and it does not contain repeated entries,
\begin{equation*}
    |\HAC(\bW)| \leq nm(1+\log(nm)) + 6nm b
\end{equation*}
when B-trees are used to represent both dictionaries implementing the mappings $H_{\bW}$ and $H_{\bW}^{-1}$.}
\end{theorem}
\begin{proof}
{
By hypothesis
each of the $nm$ symbols of $\bW$ appear exactly once, thus
$\mathcal{H} = \log (nm)$, the corresponding Huffman code $H_{\bW}$ has an average codeword length upper-bounded by $1+\log (nm)$ and at most $nm(1+\log(nm))$  bits are needed. Assuming each value in the B-tree is represented trough 1 
word ($b$ bits), each dictionary requires $3b$ bits per entry: $2b$ bits to store each pair $(z, H_{\bW}(z))$, and at most $b$ bits to store a pointer in the B-tree structure (this is overestimated since we have less pointers than keys in a B-tree).
The thesis follows.
}
\end{proof}
The upper bound provided by Fact \ref{fact1} can be reduced, as
there are methods storing a $n$-symbols Huffman code using at most $\lceil 10.75n\rceil -3$ bits \cite{Sultana12}.
Moreover, $nm(1+\log(nm)) + 6nmb$ is larger than
the number of bits required by an uncompressed matrix. For this reason,
{alternative representations should be used when the matrix is dense and no assumptions can be made on the weight distribution.}
On the other hand, in the next corollary we show that the \HAC\  becomes convenient when
{a relatively small number $k$ of distinct values are contained in $\bW$, as customary with}
quantized matrices (see Sect.~\ref{sec:compression-techniques}).}
\begin{corol}\label{corol1}
{
If $\bW$ is dense and composed of $k < nm$ distinct values,
\begin{equation*}
|\HAC(\bW)| \leq nm(1+\log k) + 6 kb .
\end{equation*}
}
\end{corol}
\begin{proof}
{In the worst case, all symbols are equally probable, the source entropy is $\mathcal{H} = \log k$, so that to represent the $nm$ values $nm(1+\log(k)) + 6 kb$ are required at most, including the overhead due to the dictionaries for the $k$ codewords.}
\end{proof}
{Corollary \ref{corol1} leads to
the following upper bound:
{
\begin{equation}\label{psi_hac}
\psi_{\HAC} \leq \frac{1+\log k}{b} + \frac{6k}{nm}
\end{equation}
}
with reference to the uncompressed matrix, where, as expected, for small $k$ the first term
is more relevant, while the second term grows faster with $k$}.

\begin{algorithm}[H]
\cornersize*{0pt}
{
\caption{Dot procedure for \HAC\ representation.}
\label{fig:pseudocode.HAC}
\textbf{Procedure} {\tt Dot$_{\HAC}$}\\
\noindent
{\tt Input}:
compressed array $\mathcal{C}_{\HAC}(\bW)$;
decoding dictionary $H_{\bW}^{-1}$;
vector $\bx \in \mathbb{R}^{n\times 1}$;
number of compressed words $N$;\\[2mm]
{\tt begin algorithm}
\begin{algorithmic}[1]
\STATE Initialize: $out:=$zeros$(n)$, $row:=1$, $col:=1$\\
\indb\inda\inda\inda\inda $sum:=0$, $rem:=null$, $oset:=0$
\FOR{$i=1$ to $N$}
  \STATE $S := $ getBinarySeq($\mathcal{C}_{\HAC}(\bW)[i]$)
  \WHILE{$[oset, rem, z]:=$ NCW($S, rem, oset)\neq null$}
    \STATE $sum:=sum+x[row]*H_{\bW}^{-1}(z),\ row:=row+1$
    \IF{$row>n$}
      \STATE $row:=1,\ out[col] := sum$
      \STATE $col:=col+1,\ sum := 0$
    \ENDIF
  \ENDWHILE
\ENDFOR
\end{algorithmic}
{\tt end algorithm}\\
{\tt Output}: $out$, that is $\bx^T \bW$.
}
\end{algorithm}

\noindent\emph{Dot product}.
{The procedure {\tt Dot$_{\HAC}$} (Algorithm~\ref{fig:pseudocode.HAC}) shows how the dot product $\bx^T\bW$ can be computed when $\bW$ is represented through \HAC.
Each compressed word in  $\mathcal{C}_{\HAC}(\bW)$ is processed sequentially, obtaining its binary representation $S$ (line $3$),
which is scanned in a loop (lines $4$-$10$) to retrieve code words. The procedure NCW gets the next code word from $S$, considering an offset $oset$ and
possibly adding unprocessed bits from the previous word (stored in $rem$).}
{NCW returns $null$ when
it is not possible to detect a code word.
This means that it took two adjacent memory words to represent the next code word. When this happens, $rem$ is updated accordingly and the processing considers the next word at the successive iteration (line $2$). Zero-padding is also handled by the procedure. Subsequently, the weight for the detected code word is computed and multiplied by the corresponding element in $\bx$, and accumulated in $sum$ (line $5$). Thus only one weight at a time is kept in main memory.}
The time complexity of the external $N$ iterations is $\mathcal{O}(|\mathcal{C}_{\HAC(\bW)}|) = \mathcal{O}(Nb)$ (line $3$), amounting in worst case (uniform frequencies of distinct weights) to $\mathcal{O}(nm\log k)$. Moreover, time complexity for lines $5$--$8$ is $\mathcal{O}(N)$, and that of executions of line $4$ is $\mathcal{O}(Nb\log k)$. Here we assume that
\begin{itemize}
    \item each time a bit is read from $S$, the current string is searched for in the dictionary (and in the worst case we need to entirely scan $S$), and
    \item searching in a dictionary of $k$ entries can be performed in $\mathcal{O}(\log k)$ time.
\end{itemize}
Summing up, the overall time complexity is $\mathcal{O}(nm\log k)$. In Sect.~\ref{sub:pardot} we describe
{how the procedure Dot$_{\HAC}$ can be reworked in order to speed up the computations.}



 \subsection{Sparse Huffman address map compression}\label{sec:sham}
{\HAC\ marginally benefits from sparsity: indeed, in such a case the space occupancy is only indirectly reduced, because now the symbol $0$ has higher frequency and the Huffman code is more compact. But when $\bW$ is sparse and very large,
we would use in any case
a high amount of memory (e.g., $10(1-s)$ GB for a $10^5\times 10^5$ matrix). To address this issue,
we propose to extend
\HAC\
into the novel \emph{sparse Huffman Address Map compression} (\sHAC),}
{in which the bit stream and the Huffman code are computed excluding the symbol $0$. More precisely, $\bW$ is represented using a bitwise CSC format, obtaining
$\bnz, \bri, \bcb$ as in Sect.~\ref{sub:CSC}: the first vector is stored using \HAC,
the others are kept uncompressed.}
{The Huffman code $H_{\bnz}$ for non-null elements is computed, obtaining by concatenation the bit stream ${\HAC(\bnz)}= H_{\bnz}(nz_1) \ldots H_{\bnz}(nz_q)$, which is stored in the array $\mathcal{C}_{\HAC}(\bnz)$ of $N_1=\lceil {|{\HAC(\bnz)}|}/{b}\rceil$ memory words.}
{Finally, we build the \sHAC\ representation of $\bW$, denoted by $\sHAC(\bW)$, as the sequence of vectors $\mathcal{C}_{\HAC}(\bnz)$, $\bri, \bcb$. The following fact establishes an upper bound for $|\sHAC(\bW)|$.}
{
\begin{theorem}(\sHAC\ worst case)\label{fact2}
{If $\bW$ contains $s n m$ non-null elements, with $s\in [0,1]$ the ratio of non-zero entries in $\bW$,
\begin{equation*}
    |\sHAC(\bW)| \leq snm\big(1+\log(snm)\big) + b(7snm + m+1)
\end{equation*}
when B-trees are used to represent both dictionaries implementing the mappings $H_{\bnz}$ and $H_{\bnz}^{-1}$.}
\end{theorem}}
\begin{proof}
{In the worst case, all $snm$ symbols are distinct, $\mathcal{H} = \log (snm)$, and the average codeword length of $H_{\bnz}$ is upper bounded by $1+\log (snm)$, thus
\begin{equation*}
    |\mathcal{C}_{\HAC}(\bnz)|\leq snm(1+\log snm).
\end{equation*}
The dictionaries $H_{\bnz}$, $H_{\bnz}^{-1}$ require $6snmb$ bits, vectors $\bri$, $\bcb$ require $b(snm+m+1)$ bits, and the thesis follows.}
\end{proof}
From {Fact~\ref{fact2} the
occupancy ratio
for \sHAC\ is such that
\begin{equation*}
    \psi_{\sHAC} \leq \frac{s(1+\log snm)}{b} + 7s + \frac{m+1}{nm},
\end{equation*}
and using the same argumentations given for \HAC, we can
show that also \sHAC\ benefits from the matrix quantization.}
\begin{corol}\label{corol2}
{In the same hypotheses of Fact~\ref{fact2}, if $\bW$ contains $k < nm$ distinct values,
\begin{equation*}
    |\sHAC(\bW)| \leq snm(1+\log k) + b(6k + snm + m + 1).
\end{equation*}}
\end{corol}
\begin{proof}
{As in Corollary~\ref{corol1}, $\mathcal{H} = \log k$, and the $snm$ values can be represented using at most $snm(1+\log k)$ bits. The result follows because $6kb$ bits are needed by the dictionaries for the $k$ codewords, and $\bri$ and $\bcb$ require $b(snm+m+1)$ bits.}
\end{proof}

{The
occupancy ratio
for \sHAC\ is in this case such that
\begin{equation}\label{psi_shac}
    \psi_{\sHAC} \leq \frac{s(1+\log k)}{b} + \frac{6k+m+1}{nm} + s,
\end{equation}
where the ratio of non-zero entries $s$ appears as last term, as well as in the first one, where it is scaled by the upper bound of $\psi_{\HAC}$, emphasizing a gain when the sparsity of $\bW$ increases.}

\begin{algorithm}[H]
{
\caption{Dot procedure for \sHAC\ representation.}
\label{fig:pseudocode.sHAC}
\textbf{Procedure} {\tt Dot$_{\sHAC}$}\\
\noindent
{\tt Input}:
compressed array $\mathcal{C}_{\sHAC}(\bnz)$;
row index vector $\bri$;
vector $\bcb$;
vector $\bx \in \mathbb{R}^{n\times 1}$;
decoding dictionary $H_{\bnz}^{-1}$;
number of compressed words $N_1$;


{\tt begin algorithm}
\begin{algorithmic}[1]
\STATE Initialize: $out:=$zeros$(n)$, $pos:=1$, $col:=1$\\
\indb\indb\inda\inda $sum:=0$, $rem:=null$, $oset:=0$\\
\FOR{$i=1$ to $N_1$}
  \STATE $S := $ getBinarySeq($\mathcal{C}_{\sHAC}(\bnz)[i]$)
  \WHILE{$[rem, oset, z]:=$ NCW($S, rem, oset)\neq null$}
    \WHILE{$cb[col+1] = pos$}
      \STATE $col := col + 1$, $out[col] := 0$
    \ENDWHILE
    \STATE $sum:=sum+x[ri[pos]]*H_{\bnz}^{-1}(z)$, $pos:= pos+1$
    \IF{$cb[col+1] = pos$}
      \STATE $out[col] := sum$
      \STATE $sum := 0,\ col := col +1$
    \ENDIF
  \ENDWHILE
\ENDFOR
\end{algorithmic}
{\tt end algorithm}\\
{\tt Output}: $out$, that is $\bx^T \bW$.
}
\end{algorithm}

{From Eqs.(\ref{psi_hac}) and (\ref{psi_shac}) it follows $\psi_{\sHAC} < \psi_{\HAC}$ when
\begin{equation*}
    s < \frac{\frac{1+\log k}{b}- \frac{m+1}{nm} }{1 + \frac{1+\log k}{b}}.
\end{equation*}
}

\noindent{\emph{Dot product}}.
Algorithm~\ref{fig:pseudocode.sHAC} shows the dot product $\bx^T\bW$ when $\bW$ is represented through \sHAC. The compressed words of $\mathcal{C}_{\sHAC}(\bnz)$ are extracted sequentially, computing each time their binary representation (line $3$) and detecting code words (lines $4$--$13$). Here, NCW is the same procedure as in {\tt Dot$_{\HAC}$}, whereas the loop at lines $5$--$7$ possibly skips empty columns. The variable $pos$ contains the position of the current element in $\bnz$. The weight for the detected codeword is computed in line $8$ and multiplied by the corresponding element in $\bx$, finally updating
the cumulative value contained in $sum$.
{The required time complexity is
\begin{inparaenum}[(i)]
    \item $\mathcal{O}(N_1 b)= \mathcal{O}(snm\log k)$ for the $N_1$ iterations in the external loop (line $3$),
    \item $\mathcal{O}(snm\log k)$ for line $4$,
    \item $\mathcal{O}(m)$ for the loop in lines $5$--$7$,
    \item $\mathcal{O}(N_1)$ for lines $8$--$12$.
\end{inparaenum}
Summing up, the overall time complexity is $\mathcal{O}(snm\log k)$. The following subsection describes how to speed up the {\tt Dot$_{\HAC}$} and  {\tt Dot$_{\sHAC}$} procedures.}

\subsection{{Speeding up the dot product of \HAC\ and \sHAC}}\label{sub:pardot}
The procedures {{\tt Dot$_{\HAC}$} and} {\tt Dot$_{\sHAC}$} can be adapted to parallel computation by exploiting the parallel nature of matrix multiplication, since any row of the left operand can undergo independently of the other rows to the dot product with the columns of the other operand. {Indeed, given two matrices $\bX$ and $\bW$, with $\bW$ compressed with either \HAC{} or \sHAC{}, the evaluation of $\bX^T\bW$ can be distributed across $q$ computing units by considering $q$ chunks of rows of $\bX$, performing $q$ dot products $\bX_i^T\bW, i \in \{1,..,q\}$,  and finally aggregating the output matrix. $\bX_i$ here denotes the submatrix composed of the rows of  $\bX$ in the $i$-th chunk. By definition, there is no data dependency between each product $\bX_i^T\bW$,
hence each chunk can be evaluated concurrently.
The pseudo-code of the procedure is shown in Algorithm~\ref{code:parallel}. Line $2$ computes the
row indices of the
matrix to be assigned to each computing unit hence $Idx$ contains $q$ tuples. In lines $3$--$9$ each computing unit (concurrently) contributes to $k$ lines of $out$ (except for the last one if $n$ is not a multiple of $q$) by, in total, evaluating $k \times t$ dot products, given $t$ the number of (encoded) columns of $\bW$.

\begin{algorithm}[H]
{
\caption{Pseudocode of the {parallel matrix multiplication} for \HAC{} representation.
}
\label{code:parallel}
\textbf{Procedure} {\tt ParDot$_{\HAC}$}\\
\noindent
{\tt Input}:
compressed array $\mathcal{C}_{\HAC}(\bW)$;
decoding dictionary $H_{\bW}^{-1}$;
expanded matrix $\bX \in \mathbb{R}^{n\times m}$;
number of compressed words $N$;
number of computing units $q$;\\
%
{\tt begin algorithm}
\begin{algorithmic}[1]
\STATE Initialize: $out:=$zeros$(n,t)$; $k := \lceil \frac{n}{q}\rceil$ \\
\STATE $Idx = [(1,k),(k$+$1,2k), \dots ,(k(q$-$1)$+$1, n)]$
\FOR{$(startIdx,endIdx)$ in $Idx$ in parallel}
  \FOR{$i = startIdx$ to $endIdx$}
    \STATE $\bx$ = $\bX[i,:]$
    \STATE $y_i$ = {\tt Dot$_{\HAC}$}$(\mathcal{C}_{\HAC}(\bW), H_{\bW}^{-1}, \bx, N)$
    \STATE $out[i,:]$ = $y_i$
  \ENDFOR
\ENDFOR
\end{algorithmic}
{\tt end algorithm}\\
\noindent
{\tt Output}: $out$, that is $\bX^T \bW$.
}
\end{algorithm}

 Note that the analogous parallel version for \sHAC\ can be derived
 by adding as inputs  $\mathcal{C}_{\sHAC}(\bnz)$, the row index vector $\bri$ and the vector $\bcb$, and by invoking in line 6 $y_i$ = {\tt Dot$_{\sHAC}$}$(\mathcal{C}_{\sHAC}(\bW)$, $\bri$, $\bcb$, $H_{\bW}^{-1}$, $\bx$, $N$).}
{To test the viability of this approach, we implemented this strategy by means of CPU multi-threading. However, the Python global interpreter lock poses severe limitations to multi-threaded execution.
To overcome this limitation, {\tt ParDot$_\HAC$} has been implemented in C++, which provides no limitations to
multi-threading, while the rest of the code is executed within Python, using the Pybind11}\footnote{Available at: \url{https://github.com/pybind/pybind11}} library for Python / C++ interoperability.

\section{Experiments and Results}\label{sec:experiments}
{In this section we empirically compare the illustrated techniques. The experiments considered four datasets and two uncompressed neural networks, as detailed here below.}

\subsection{Data}\label{sub:data}
\begin{itemize}
    \item \emph{Classification}.
    MNIST~\cite{MNIST}, a benchmark of handwritten digits,
    {containing a train set of 60K 28x28 grayscale images and a test set of 10K analogous images;}
    CIFAR-10~\cite{Krizhevsky09learningmultiple}, a dataset of
    {50K (train set) + 10K (test set)}
    32x32 color images.
    Both datasets refer to ten classes (one for each digit) and their labels are balanced.

    \item \emph{Regression}. DAVIS~\cite{Davis11} and KIBA~\cite{KIBA14},
    {datasets containing the evaluation of the}
    affinity between drugs (ligands) and targets (proteins),
    respectively represented using the amino acid sequence and the SMILES (Simplified Molecular Input Line Entry System)
    {string encoding.}
    DAVIS and KIBA contain, respectively, $442$ and $229$ proteins, $68$ and $2111$ ligands, $30056$ and  $118254$ total interactions {between them, with \nicefrac{1}{6} of the data composing the test set}.
\end{itemize}
\subsection{Benchmark models}\label{sub:models}
{We used to publicly available pre-trained, top-performing CNN models:}
{
\begin{inparaenum}[(i)]
\item \emph{VGG19} \cite{Simonyan15}, consisting of $16$ convolutional layers followed by a fully-connected (FC) block, in turn containing two hidden layers of $4096$ neurons each, and a softmax output layer\footnote{\url{https://github.com/BIGBALLON/cifar-10-cnn}.}, trained on CIFAR-10 and MNIST datasets;
\item \emph{DeepDTA} \cite{DeepDTA},  having two separate blocks for proteins and ligands, both containing three convolutional layers followed by a max pool layer and merged in a FC block consisting of three hidden layers respectively containing $1024$, $1024$, $512$ units, and a single-neuron output layer\footnote{\url{ https://github.com/hkmztrk/DeepDTA}.}.
\end{inparaenum}
}
Using pre-trained networks allows a fair analysis of compression and storage techniques, without introducing potential biases in the model selection and training procedures. Moreover, when fine-tuning weights after quantization, we preserved the same training configuration set up by the model proponents in their original work.

\subsection{Evaluation metrics}\label{sub:metrics}
{We performed comparisons focusing on the following metrics:
\begin{inparaenum}
    \item \textit{Accuracy} for classification and \textit{MSE} for regression (as in original papers) or the \textit{difference} $\Delta_\mathrm{perf}$ between performances of compressed and uncompressed models;
    \item ratio $time$ between evaluation times of uncompressed and compressed model, and
    \item occupancy ratio $\psi$ (cfr.\ Sect.~\ref{sub:CSC}).
\end{inparaenum}
}
{When only partly compressing the NN, time and space performance only account for the actually compressed layers.}
{The rest of the paper assesses the effectiveness of compression techniques in three scenarios, namely compressing: \begin{inparaenum}
    \item only FC layers,
    \item only convolutional layers, and
    \item both layer types.
\end{inparaenum}
}
\subsection{Software implementation}
The code retrieved for baseline NNs was implemented in Python 3, using Tensorflow
and Keras.
{We used the same environment for implementing compression and retraining (although the parallel dot procedures are written in C++). The software is distributed as a standalone Python 3 package\footnote{Source code, datasets and trained baseline networks are  available at \url{https://github.com/AnacletoLAB/sHAM}.}}.

\subsection{Results}\label{sub:res}
{We conducted multiple evaluations, analyzing the proposed compression and storage methodologies from different points of view, as described in the following dedicated subsections.
\subsection{Preliminary results from previous studies} \label{subsub:icpr}
{{In this subsection we summarize the results obtained when only compressing FC layers via
CWS and PWS,
separately considering each layer~\cite{HAMICPR}.}
}
\paragraph{Compression techniques setup}
{the schemes corresponding to pruning (Pr), CWS, PWS, Pr-CWS, Pr-PWS have been been tested on data and models described in Sects.~\ref{sub:data} and~\ref{sub:models} using the hyperparamter tuning described here below.}
\begin{itemize}
\item \emph{Pruning}.
{The percentile level $p$ was chosen in the set $\{30,  40,  50, 60,  70,  80,  90,  95,  96,  97,  98, 99\}$\footnote{levels smaller than $50$, although not guaranteeing occupancy $<1$,  were included because potentially useful in the combinations Pr-CWS and Pr-PWS.}; }
\item \emph{CWS}.
{The number $k$ of representatives for VGG19 was selected in $\{2, 32, 128, 1024 \}$ for the first two FC layers of and between $2$ and $32$ for the (smaller) output layer; as DeepDTA is more compact, we set $k \in \{2, 32, 128 \}$ in the three FC layers and $k \in \{ 2, 32 \}$ for the output layer.
}
\item \emph{PWS}.
{To have a fair comparison, $k$ was set as in CWS.}
\item \emph{Pr-X}. The combined application of pruning followed by the quantization $X \in \{\mathrm{CWS}, \mathrm{PWS}\}$ was tested in two variants: a)
selection of best $p$ in terms of $\Delta_\mathrm{perf}$,
and then tuning of $X$ as in previous points;
b) the vice-versa.
\end{itemize}

\paragraph{Fine-tuning of compressed weights}
post-compression retraining was done using the same configuration as in original training.
Data-based tuning was applied only to learning rate
($3\cdot10^{-4}$ for pruning, $10^{-3}$ and $10^{-4}$ for PWS, CWS, and combined schemes), and maximum number of epochs, set to  $100$. As explained in Sect.~\ref{sub:ham}, in the experiments using only pruning the CSC representation is adopted. \\

\paragraph{Performance assessment}
{Table~\ref{original_res} reports the testing performance of uncompressed models as a baseline. The top performance for each compression technique, and its configuration, is shown in Supplementary Table S1, whereas
{the configuration improving the baseline (when existing) having the smallest memory requirement}
is shown in Supplementary Table S2. WQ performed better than pruning for classification}, with PWS and CWS having the top performance on MNIST and CIFAR-10, respectively. Remarkably, all techniques outperformed the baseline, while exhibiting
{effective}
compression rates.
Regression behaved similarly, however here
pruning was preferable, and PWS never improved the baseline on KIBA.
Improvements were particularly
{interesting}
on DAVIS (up to around $30\%$ of baseline).
The largest compression rate (preserving accuracy) was achieved on the biggest net, VGG19, attaining more than $150\times$ on CIFAR-10 (Pr/PWS-b and \sHAC).
When applied to DeepDTA, Pr/PWS-a improved the baseline MSE of $ 17.1\% $, while compressing around $18\times$.
\begin{table}[]
\centering
\caption{Testing performance of original non-compressed models. \emph{Performance} shows  accuracy for MNIST/CIFAR-10 and  MSE for  KIBA/DAVIS. \emph{Time} is the overall testing time.}
\label{original_res}
\begin{tabular}{|l|l|rr|}
\hline
\rowcolor[HTML]{B6D7A8}
\multicolumn{1}{|c|}{\cellcolor[HTML]{B6D7A8}\textbf{Net}} & \multicolumn{1}{c|}{\cellcolor[HTML]{B6D7A8}\textbf{Dataset}} & \multicolumn{1}{c|}{\cellcolor[HTML]{9FC5E8}\textbf{Performance}} & \multicolumn{1}{c|}{\cellcolor[HTML]{9FC5E8}\textbf{Time (s)}} \\ \hline
 & MNIST & 0.9954 & 0.888 \\ \cline{2-2}
\multirow{-2}{*}{VGG} & CIFAR-10 & 0.9344 & 0.897 \\ \hline
 & KIBA & 0.1756 & 0.175 \\ \cline{2-2}
\multirow{-2}{*}{DeepDTA} & DAVIS & 0.3223 & 0.040 \\ \hline
\end{tabular}
\end{table}


{A summary of space occupancy, time ratio, and testing performance for all used hyper-parameter configurations is shown in Supplementary Fig.~S1,
where \sHAC\ is used, except for techniques producing denser matrices, where \HAC\ was more convenient. The time reported is relative to the sequential dot procedure for \HAC\ and \sHAC}. CWS and PWS combinations are reported in increasing order: first the ones with $k=2$ in the first layer (label $2$), then those with $k=32$ in the first layer (label $32$) and so on. On CIFAR-10 and DAVIS, most compression techniques outperformed the baseline, and this was likely due to overfitting, since on training data they show similar results.
{On MNIST and KIBA, although the baseline was
only seldom improved,
the compressed model used much less parameters, confirming the trend obtained in \cite{Han15}. Binary quantization ($k=2$) achieved the lowest $\psi$, yet a worse performance, whereas already with $k=32$ the baseline was improved on almost all datasets}. \sHAC\ occupancy, as expected, decreased inversely with $p$, along with the time ratio, approaching in turn to $1$ (same testing time). The high time ratios, for some configurations, reflected the fact that the {\tt dot} procedure was slower than the \emph{Numpy} {\tt dot} used by baseline and leveraging parallel computation.

{As pointed out
in~\cite{HAMICPR}, these results did not highlight a compression technique better that the remaining ones. Weight pruning seemed to be preferable for regression, whereas quantization performed better in the classification setting. Overall, the most remarkable achievement can be considered the fact that compression techniques providing the lowest occupancy, i.e., those combining weight pruning  and quantization, still achieved competitive or better performance than the baseline.}

\subsection{State-of-the-art comparison of \HAC\ and \sHAC}\label{sub:soa}
\begin{figure*}[!th]
\centering
\includegraphics[width=0.765\textwidth]{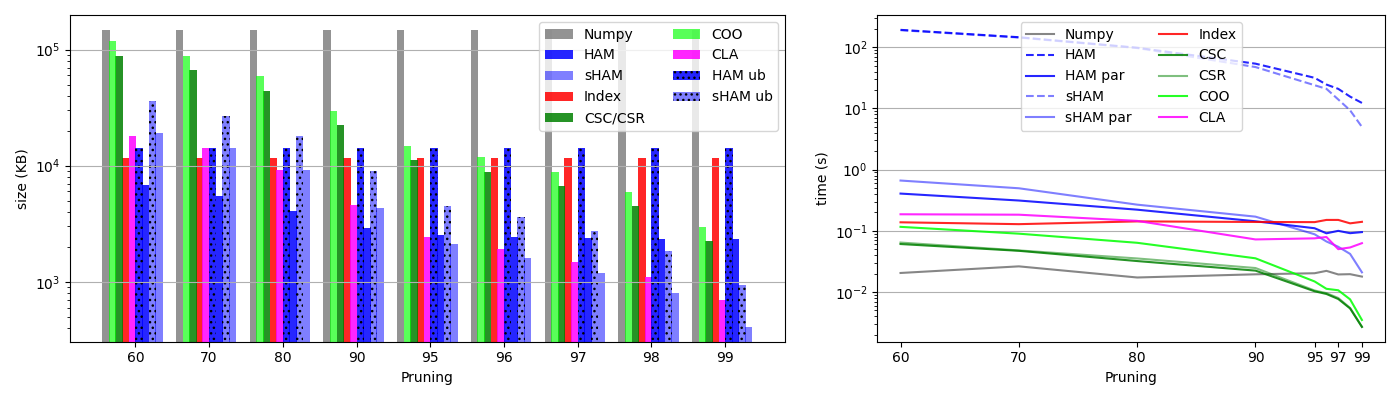}
\includegraphics[width=0.765\textwidth]{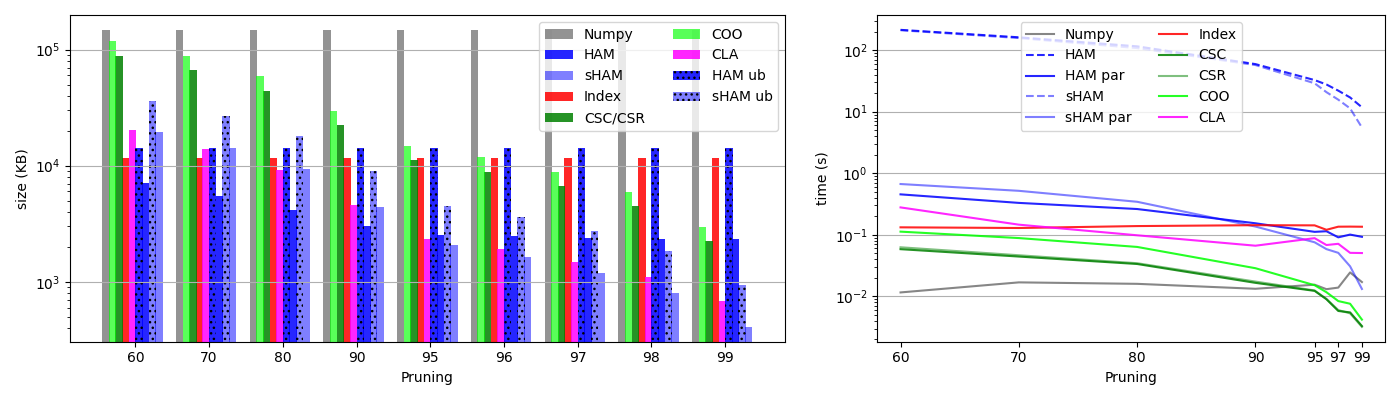}
\caption{Average execution times in seconds (graphs on the right) for performing $8$ vector matrix product, and memory footprint in KiloBytes (graphs on the left) for storing the compressed representation of the three fully connected layer weight matrices of the VGG19 model. Top rows refer to the model trained on CIFAR-10 dataset, bottom row on MNIST dataset. All matrices were subjected to different degrees of pruning and were quantized by weight sharing (CWS) with 32 values. Dotted bars represent the upper bound of \HAC{} and \sHAC{} sizes evaluated with Corollaries \ref{corol1} and \ref{corol2}.  Times are reported in logarithmic scale.}\label{fig:cifarMnistTimeSpace}
\normalsize
\end{figure*}

{In order to assess the efficacy of the proposed formats, a suite of state-of-the-art matrix compression methods has been considered; 
{such methods allow the execution of linear algebra operations, such as matrix-vector multiplication, directly on the compressed representation.}
{The comparison takes into account both space reduction and multiplication time on the compressed format, and includes the following methods:
{
\begin{itemize}
\item \emph{Compressed Linear Algebra} (CLA), a compendium of effective column compression schemes, cache-conscious operations, and {a} sampling-based compression algorithm to select the compression scheme more suitable for each column or group of columns;
this approach achieves performance close to the uncompressed case and compression ratios similar to heavyweight formats such as Gzip, but better
than lightweight formats like Snappy~\cite{CLA};
\item three baseline formats: CSC (cfr.\ Sect.~\ref{sub:CSC}), \emph{Compressed Sparse Row Format} (CSR), analogous to CSC but storing the column indices of nonzero values, instead of their rows~\cite{CSR}, and \emph{Coordinate list} (COO), storing non zero elements along with their row and column indices;
\item \emph{index map} (IM)~\cite{Han15} (see Sect.~\ref{subsub:CWS}), proposed expressly in the context of neural network compression.
\end{itemize}
}
As baseline reference for time, we also evaluated
{the \texttt{dot} function}
of the \emph{Numpy} library, which is also used by the dot product of the IM method, along with the additional {access time}
to the vector of full precision weights.
CLA is implemented in Java, and the code is available upon request from the authors.}
CSC, CSR and COO
implementations
are taken from the \textit{Scipy} library. \HAC\ and \sHAC\ dot is called within Python, but its underlying implementation is done in C++ (cfr.\ Sects~\ref{sub:ham} and~\ref{sec:sham}).
The test is performed by calculating the dot product between a sparse matrix and a dense vector, without expanding the matrix.
For this reason, we excluded from this evaluation other compression techniques that needed to expand the matrix to perform fine-tuning after quantizing,
such as the \emph{Tensorflow Lite Converter}}\footnote{\url{https://www.tensorflow.org/lite/convert}}.

{The evaluation involves the weight matrices of the three FC layers composing VGG19, trained on CIFAR-10 and MNIST, after such layers underwent pruning (various levels, from $60\%$ to $99\%$) and CWS with $k = 32$ and $256$ (as examples of less and more marked quantization)---as in average CWS tends to perform better than PWS on this data, as shown in Sect.~\ref{subsub:icpr}). 
Lower pruning rates were meaningless,
since we are evaluating compression for sparse matrices}.
The dimensions of matrices for the three FC layers are $512 \times 4096$, $4096 \times 4096$ and $4096 \times 10$.
The resulting size is the sum of the memory footprint of all involved structures (vectors, lists, dictionaries, etc.).}

{
Comparison of the dot product performance was done generating
$8$ dense vectors {according to a uniform distribution in $[0,1]$,} and executing the operation
{using as arguments each generated vector and the compressed matrix.}
The overall time is the sum of the $8 \times 3$ dots ($8$ dot products for each matrix). Executions were multi-threaded ($8$ threads per dot product).}

Results are summarized on the graphs of Fig.~\ref{fig:cifarMnistTimeSpace} for $k=32$ and in Fig.~S2 for $k=256$. Time is measured in seconds, size in Kilobytes.
{On both datasets, with lower pruning ($p\in[60,90]$) and $k=32$, \HAC\ shows the highest compression rate, from $\simeq 21\times$ ($p=60$) to $\simeq 50\times$ ($p=90$), whereas when the matrices get highly sparse ($p \geq 90$), \sHAC\ compresses the most---till around $361\times$ ($p=99$). IM, using 1 byte for each entry of the matrix when $k=32$, has better compression ratios than CLA and \sHAC\ only up to $p=70$. This is reasonable, since this method does not
exploit
the matrix sparsity (0 is just one of the $k$ values to be stored).
}
CSC, CSR and COO behave in a similar way, and always occupy more than CLA and \HAC/\sHAC, and less than IM only when $p \geq 0.95$.
On the other hand, such formats allow the fastest dot product (even faster than \emph{Numpy} when $p > 90$), as \emph{Scipy} is designed to perform dot products very efficiently. Among the remaining methods, with high pruning ($p > 90$) \HAC\ and \sHAC\ are faster than the other methods (except for \emph{Numpy} dot).
Below that boundary, IM is faster than CLA  for $p \leq 70$.
\emph{Numpy} often performs the fastest dot product, but it is useless since it operates on the uncompressed matrix.
Analogous trends are verified for $k=256$, with index map that compares little more favourably, since its space and time are almost the same (it still uses only 1 Byte for entries of $\bI$), while the other methods slightly increase both their space and time requirements.

Summing up, \HAC\ and \sHAC\ compress the most (\HAC\ for $p\leq 90$, \sHAC\ in the other cases), and \sHAC\ is the best in time for $p \geq 95$ (it is still faster than the \emph{Numpy} dot on the MNIST---$p=99$ case). For
low/medium pruning rates, IM and CLA exhibit a higher time efficiency than \HAC\ and \sHAC; these features are well expected, as our methods have been developed and optimized for obtaining state-of-the-art compression ratios when pruning and quantization is applied, at the cost of requiring a quite consistent pruning to show fast dot products.
Between IM and CLA, the former is preferable for low pruning rates in terms of both time and space.

Both figures also show the upper bounds introduced in Sects~\ref{sub:ham} and \ref{sec:sham}, evaluated using Corollaries~\ref{corol1} and~\ref{corol2}. Both techniques show an actual size consistently lower than the upper bound: on average, \sHAC{} is twice as lower, while \HAC{} ranges from $2\times$ to $6\times$ lower, depending on pruning ratio. So, the assumption of weights having the same frequency (used to compute upper bounds) is farther from being realistic,
and in practice the proposed formats are definitely more effective.
Tighter bounds might be found if more realistic assumptions are considered, which could be the subject of future studies.
\subsection{Evaluation of global and per-layer quantization}\label{sub:unifiedQuant}

The assessment performed in Sect.~\ref{subsub:icpr} made use of a \emph{non-unified} quantization, selecting a specific $k$
per layer and quantizing layers separately.  In essence, all weight matrices could be quantized at the same time using a unique value of $k$, leading to a variant that we name \emph{unified} quantization. Here we compare these two approaches. {To this end, the setting used in Sect.~\ref{subsub:icpr} is preserved, that is compressing only dense layers and measuring} performance in terms of accuracy/MSE and occupancy ratio. It is worth noting that
the non-unified case
has higher flexibility, counterbalanced by an increase of memory overhead, as each layer requires its own dictionaries. On the other hand, the unified case entails only one dictionary
but the
choice of $k$ represents an `one size fits all' solution, which reduces versatility
and might impact on performance.

First, the best models in terms of prediction performance from Supplementary Table S1 have been considered, meanwhile using the same compression techniques.
{The value for $k$ in the unified case was selected by summing the values of $k$ in the layers of the selected models in the non-unified case.}
The resulting four experiments are summarized in Table~\ref{tab:comparision-unique-vs}, where the modalities exploting the unified choice of $k$ are marked as ``uCWS'' and ``uPWS''.
\begin{table}[!t]
\centering
\caption{Comparison between unified  and non-unified quantization. $\psi$ refers to \HAC\ format. Baseline performance in brackets.}
\label{tab:comparision-unique-vs}
\resizebox{1 \columnwidth}{!}{
\begin{tabular}{|c|l|c|rr|}
\hline
\rowcolor[HTML]{B6D7A8}
\textbf{Net-Dataset} & \multicolumn{1}{c|}{\cellcolor[HTML]{B6D7A8}\textbf{Type}} & \textbf{Config} & \multicolumn{1}{c|}{\cellcolor[HTML]{9FC5E8}\textbf{Perf}} & \multicolumn{1}{c|}{\cellcolor[HTML]{9FC5E8}\textbf{$\psi$}} \\ \hline
 & CWS & 128-32-32 & 0.9957 & 0.3210 \\ \cline{2-2}
 & uCWS & 192 & 0.9957 & 0.2344 \\ \cline{2-2}
 & PWS & 32-32-2 & 0.9958 & 0.3090 \\ \cline{2-2}
\multirow{-4}{*}{\begin{tabular}[c]{@{}c@{}}VGG19-MNIST\\(0.9954)\end{tabular}} & uPWS & 66 & 0.9955 & 0.1857 \\ \hline
 & CWS & 32-32-2 & 0.9371 & 0.3060 \\ \cline{2-2}
 & uCWS & 66 & 0.9370 & 0.1856 \\ \cline{2-2}
 & PWS & 32-2-32 & 0.9363 & 0.0910 \\ \cline{2-2}
\multirow{-4}{*}{\begin{tabular}[c]{@{}c@{}}VGG19-CIFAR-10\\ (0.9344)\end{tabular}} & uPWS & 66 & 0.9366 & 0.1857 \\ \hline
 & CWS & 128-128-32-2 & 0.1679 & 0.3900 \\ \cline{2-2}
 & uCWS & 290 & 0.1609 & 0.2516 \\ \cline{2-2}
 & PWS & 32-128-128-32 & 0.1761 & 0.4250 \\ \cline{2-2}
\multirow{-4}{*}{\begin{tabular}[c]{@{}c@{}}DeepDTA-KIBA\\ (0.1756)\end{tabular}} & uPWS & 320 & 0.1631 & 0.2642 \\ \hline
 & CWS & 128-2-128-2 & 0.2320 & 0.2120 \\ \cline{2-2}
 & uCWS & 260 & 0.2291 & 0.2496 \\ \cline{2-2}
 & PWS & 128-32-32-32 & 0.2430 & 0.3240 \\ \cline{2-2}
\multirow{-4}{*}{\begin{tabular}[c]{@{}c@{}}DeepDTA-DAVIS\\ (0.3223)\end{tabular}} & uPWS & 224 & 0.2253 & 0.2469 \\ \hline
\end{tabular}
}
\end{table}
The results point up that such variants tend to compress more than their non-unified counterparts, having in some cases negligible performance decay (classification), or even improvements (regression). Specifically, uCWS compresses up to $1.64\times$ more than CWS (VGG19-CIFAR-10), and uPWS up to $1.66\times$ more than PWS (VGG19-MNIST), while almost preserving the accuracy ($-0.03\%$ and $-0.01\%$, respectively). Besides, uCWS improves the MSE up to $4\%$ on DeepDTA-KIBA (while compressing $1.55\times$ more), and uPWS up to the $8\%$ on DeepDTA-DAVIS, even with a compression ratio of $1.31\times$ higher.
Nevertheless,
the above mentioned flexibility of non-unified variants let them compress
more than the unified counterparts, mainly when they can use lower $k$ values in large layers. For instance, PWS has less than half occupancy compared to uPWS on VGG19-CIFAR-10 by using only two distinct weights for the central hidden layer of VGG19 (the largest one).
In order to confirm these observations, we repeated this experiment using the non-unified models reported in Supplementary Table S2 (best compression preserving the baseline performance), as shown in Supplementary  Table~S3. The trend of Table \ref{tab:comparision-unique-vs} is preserved, although here the
higher flexibility for non-unified methods is more marked (indeed they always compress more on classification). For this reason and to better evaluate the behavior of unified versions, we additionally tested them by varying $k \in \{2^i | 1 \leq i \leq 7\}$,
and reporting the setting achieving the best compression ratio among those performing at least as good as the non-unified counterparts.  
With this ``non-constrained'' setting, they show in most cases an occupancy much lower than the non-unified variants (e.g., $4\times$ smaller on VGG19-MNIST data). Notwithstanding, even in this case,
the non-unified variants can sometimes have a better compression (like on VGG19-MNIST data for the CWS method).

Overall, although both variants can be better than the corresponding counterparts in some specific cases,
unified approaches tend to achieve both better compression and higher performance; thus, hereafter this variant will be used.
\subsection{Comparison  of  quantization  techniques}\label{sub:SOA}
{
As unified quantization is preferable, it
has been applied to compare also the two quantization strategies described in Sect.~\ref{sub:WSquant}, namely UQ and ECSQ (uUQ and uECSQ for the unified version). The setting used up to now (compressing only FC layers and measuring accuracy/MSE and occupancy ratio) has been maintained. All methods have been evaluated by varying  $k$ in $\{2, 16, 32, 64, 128, 256\}$.
{\HAC\ was used to store the compressed matrices, since it has proven to be more effective in this setting.}
{
Parameters $\lambda$ (uECSQ) and $\delta$ (uUQ)
have been tuned to give in output the number $k$ of desired clusters, whereas to reduce the already massive set of experiments, we have set $d=0$ for uUQ}. 
\begin{table}[]
\centering
\caption{Summary of the perfomance of VGG19 trained and tested over the MNIST and CIFAR-10 datasets, after applying different quantization techniques on the Dense layers. \emph{Perf} measured as Accuracy. Values in the $\psi$ column report the ratio between the \HAC{} compressed and uncompressed network sizes. In brackets the Accuracy of the baseline.}
\label{tab:summary-perf-VGG}
\begin{tabular}{cl|rr|rr|}
\cline{3-6}
\multicolumn{1}{l}{} &  & \multicolumn{2}{c|}{\cellcolor[HTML]{B4A7D6}\textbf{\begin{tabular}[c]{@{}c@{}}MNIST\\ (0.9954)\end{tabular}}} & \multicolumn{2}{c|}{\cellcolor[HTML]{B4A7D6}\textbf{\begin{tabular}[c]{@{}c@{}}CIFAR-10\\ (0.9344)\end{tabular}}} \\ \hline
\rowcolor[HTML]{9FC5E8}
\multicolumn{1}{|c|}{\cellcolor[HTML]{B6D7A8}\textbf{$k$}} & \multicolumn{1}{c|}{\cellcolor[HTML]{B6D7A8}\textbf{Method}} & \multicolumn{1}{c|}{\cellcolor[HTML]{9FC5E8}\textbf{Perf}} & \multicolumn{1}{c|}{\cellcolor[HTML]{9FC5E8}\textbf{$\psi$}} & \multicolumn{1}{c|}{\cellcolor[HTML]{9FC5E8}\textbf{Perf}} & \multicolumn{1}{c|}{\cellcolor[HTML]{9FC5E8}\textbf{$\psi$}} \\ \hline
\multicolumn{1}{|c|}{\cellcolor[HTML]{FFFFFF}} & uCWS & 0.2266 & 0.0313 & 0.9355 & 0.0313 \\ \cline{2-2}
\multicolumn{1}{|c|}{\cellcolor[HTML]{FFFFFF}} & uPWS & 0.9951 & 0.0313 & 0.9363 & 0.0313 \\ \cline{2-2}
\multicolumn{1}{|c|}{\cellcolor[HTML]{FFFFFF}} & uUQ & 0.2213 & 0.0313 & 0.1981 & 0.0313 \\ \cline{2-2}
\multicolumn{1}{|c|}{\multirow{-4}{*}{\cellcolor[HTML]{FFFFFF}2}} & uECSQ & 0.9901 & 0.0461 & 0.9368 & 0.0472 \\ \hline
\multicolumn{1}{|c|}{} & uCWS & 0.9954 & 0.1215 & 0.9366 & 0.1179 \\ \cline{2-2}
\multicolumn{1}{|c|}{} & uPWS & 0.9953 & 0.1212 & 0.9368 & 0.1212 \\ \cline{2-2}
\multicolumn{1}{|c|}{} & uUQ & 0.2159 & 0.0314 & 0.1991 & 0.0316 \\ \cline{2-2}
\multicolumn{1}{|c|}{\multirow{-4}{*}{16}} & uECSQ & 0.9957 & 0.0949 & 0.9369 & 0.0801 \\ \hline
\multicolumn{1}{|c|}{} & uCWS & 0.9957 & 0.1467 & 0.9365 & 0.1513 \\ \cline{2-2}
\multicolumn{1}{|c|}{} & uPWS & 0.9955 & 0.1544 & 0.9365 & 0.1545 \\ \cline{2-2}
\multicolumn{1}{|c|}{} & uUQ & 0.2239 & 0.0322 & 0.9370 & 0.0355 \\ \cline{2-2}
\multicolumn{1}{|c|}{\multirow{-4}{*}{32}} & uECSQ & 0.9955 & 0.1338 & 0.9366 & 0.1241 \\ \hline
\multicolumn{1}{|c|}{} & uCWS & 0.9957 & 0.1835 & 0.9364 & 0.1836 \\ \cline{2-2}
\multicolumn{1}{|c|}{} & uPWS & 0.9955 & 0.1867 & 0.9365 & 0.1867 \\ \cline{2-2}
\multicolumn{1}{|c|}{} & uUQ & 0.8908 & 0.0397 & 0.9362 & 0.0498 \\ \cline{2-2}
\multicolumn{1}{|c|}{\multirow{-4}{*}{64}} & uECSQ & 0.9956 & 0.1841 & 0.9366 & 0.1359 \\ \hline
\multicolumn{1}{|c|}{} & uCWS & 0.9956 & 0.2134 & 0.9364 & 0.2162 \\ \cline{2-2}
\multicolumn{1}{|c|}{} & uPWS & 0.9954 & 0.2184 & 0.9363 & 0.2184 \\ \cline{2-2}
\multicolumn{1}{|c|}{} & uUQ & 0.9955 & 0.0559 & 0.9363 & 0.0736 \\ \cline{2-2}
\multicolumn{1}{|c|}{\multirow{-4}{*}{128}} & uECSQ & 0.9958 & 0.1953 & 0.9364 & 0.1787 \\ \hline
\multicolumn{1}{|c|}{} & uCWS & 0.9957 & 0.2477 & 0.9367 & 0.2468 \\ \cline{2-2}
\multicolumn{1}{|c|}{} & uPWS & 0.9955 & 0.2500 & 0.9363 & 0.2500 \\ \cline{2-2}
\multicolumn{1}{|c|}{} & uUQ & 0.9953 & 0.0971 & 0.9364 & 0.1154 \\ \cline{2-2}
\multicolumn{1}{|c|}{\multirow{-4}{*}{256}} & uECSQ & 0.9957 & 0.2283 & 0.9367 & 0.2395 \\ \hline
\end{tabular}
\end{table}
Results are summarized in {Tables \ref{tab:summary-perf-VGG} and S4}. In classification, we notice that the predicting performance of all models is often comparable with the baseline, with rare exceptions in which performance drastically worsens (e.g., uUQ with low $k$).
However, for $k > 64$ (MNIST) or $k > 32$ (CIFAR-10), uUQ becomes competitive in accuracy, outdoing other methods' occupancy ratio (up to around $4 \times$ lower). The same considerations hold for regression, in which the uUQ performance boundary is even lower ($k > 2$ for KIBA and $k > 32$ for DAVIS). When $k \leq 64$, uECSQ often exhibits the top performance with uPWS (classification) and uCWS (regression), though usually yielding bigger occupancy than uECSQ. Moreover, uECSQ tend to compress more than uCWS and cPWS.

Summarizing, yet not exhaustively, these experiments suggest to adopt uUQ when enough distinct weights can be used, whereas uECSQ should be employed in the remaining cases, in which uPWS  and uCWS (respectively on classification and regression) are valid alternatives.

Finally, to provide an insight on how pruning interacts with quantization, we repeated the experiments preposing a pruning stage to quantization as done in Sect.~\ref{subsub:icpr}, varying $p$ in $\{30,  40,  50, 60,  70,  80,  90,  95,  96,  97,  98, 99\}$, and reporting the top performances (Supplementary Table~S5) and the best occupancy configurations ensuring baseline performance (Supplementary Table~S6). In terms of best performance, the methods perform similarly (better than the baseline),
{with uUQ compressing}
much more than all other methods for classification, substantially confirming the results obtained with no pre-pruning. In terms of best occupancy
uECSQ becomes again competitive with uUQ, especially on CIFAR-10 and KIBA. In this setting, even uCWS and uPWS obtain good results (respectively top performance on KIBA and lowest occupancy on DAVIS). On the whole, the tendency shown in Tables~\ref{tab:summary-perf-VGG} and S4 is confirmed when pruning before sharing weights, with the precious benefit of similarly performing while decreasing the occupancy by almost one order of magnitude.
}

\subsection{Compressing only convolutional  layers}
\label{cnn-compr}

{Up to this point only FC layers have been compressed, and in this section we want instead to evaluate the performance of pruning and quantization applied to the convolutional layers of pre-trained models. The aim is obtaining useful information for the final experiment, where both convolutional and FC layers will be compressed simultaneously. Only the performance (accuracy or MSE) is thereby evaluated here, to detect to most meaningful compression configurations on these layers.

}
\subsubsection{Weight pruning}
here weight pruning has been applied only to the weight tensors of convolutional stages.
Results are summarized in Table \ref{tab:cnn-pruning}.
Pruning helps also in this case, with performance improvements w.r.t.\ the baseline (row $p=0$) up to  $p=70$ for MNIST and KIBA, $p=10$ for CIFAR and $p=80$ for DAVIS. {The top performance achieved are slightly worse than those obtained
in Sect.~\ref{subsub:icpr};
notwithstanding, it is not possible to achieve the same pruning percentile used on FC layers, which is quite expected, since convolutional layers are responsible for input scan and elaboration. Hence, pruning on these layers cannot always be increased till levels making beneficial the use of \sHAC, without having a sensible performance loss. Further, with no more than $10\%$ of pruning (CIFAR-10), even \HAC\ would compress less, while increasing the dot procedure time (see Fig.~\ref{fig:cifarMnistTimeSpace}). Finally, this deterioration would also be amplified by the application of weight sharing, still fostering a reduction of pruning percentile. In such a setting, the adoption of our storage formats should be based on the maximum levels of pruning applicable without having an excessive performance loss, as discussed also here below.

\begin{table}[]
\centering
\caption{Summary of the testing performance (accuracy for VGG19 and MSE for DeepDTA) of the networks after applying pruning to convolutional layers. Column $p$ is the level of pruning.}\label{tab:cnn-pruning}
\begin{tabular}{c|cccc|}
\cline{2-5}
\multicolumn{1}{l|}{} & \multicolumn{2}{c|}{\cellcolor[HTML]{9FC5E8}\textbf{VGG19}} & \multicolumn{2}{c|}{\cellcolor[HTML]{9FC5E8}\textbf{DeepDTA}} \\ \cline{1-1}
\rowcolor[HTML]{9FC5E8}
\multicolumn{1}{|c|}{\cellcolor[HTML]{A8D08D}$p$} & \textbf{MNIST} & \multicolumn{1}{c|}{\cellcolor[HTML]{9FC5E8}\textbf{CIFAR}} & \textbf{KIBA} & \textbf{DAVIS} \\ \hline
\multicolumn{1}{|c|}{0} & 0.9954 & 0.9344 & 0.1756 & 0.3223 \\
\multicolumn{1}{|c|}{10} & 0.9957 & 0.9355 & 0.1561 & 0.2220 \\
\multicolumn{1}{|c|}{20} & 0.9957 & 0.9341 & 0.1565 & 0.2233 \\
\multicolumn{1}{|c|}{30} & 0.9957 & 0.9337 & 0.1566 & 0.2238 \\
\multicolumn{1}{|c|}{40} & 0.9957 & 0.9333 & 0.1576 & 0.2218 \\
\multicolumn{1}{|c|}{50} & 0.9955 & 0.9289 & 0.1571 & 0.2237 \\
\multicolumn{1}{|c|}{60} & 0.9956 & 0.9255 & 0.1577 & 0.2224 \\
\multicolumn{1}{|c|}{70} & 0.9951 & 0.9179 & 0.1600 & 0.2234 \\
\multicolumn{1}{|c|}{80} & 0.9944 & 0.9084 & 0.2223 & 0.2433 \\
\multicolumn{1}{|c|}{90} & 0.9917 & 0.8802 & 0.3139 & 0.3492 \\
\multicolumn{1}{|c|}{95} & 0.9907 & 0.7950 & 0.3692 & 0.4136 \\
\multicolumn{1}{|c|}{96} & 0.9909 & 0.7608 & 0.3796 & 0.4753 \\
\multicolumn{1}{|c|}{97} & 0.9903 & 0.6910 & 0.4067 & 0.5180 \\
\multicolumn{1}{|c|}{98} & 0.9882 & 0.6154 & 0.4576 & 0.5350 \\
\multicolumn{1}{|c|}{99} & 0.9852 & 0.5204 & 0.5446 & 0.6548 \\ \hline
\end{tabular}
\end{table}

\subsubsection{Quantization via weight sharing}\label{subsub:quantCNN}
{the convolutional blocks have been compressed via the four quantization methods compared so far. The unified variant (cfr.\ Sect.~\ref{sub:unifiedQuant}) is considered, sharing weights and dictionaries globally across layers. The results for $k\in \{32, 64, 128, 256\}$ are shown in Table~S7, where too low values for $k$ have not been reported because they showed poor results. For classification, quantizing convolutional layers clearly leads to less effective models: indeed, the baseline is almost never improved. Moreover, the difficulty of uUQ with low values of $k$ observed in Sect.~\ref{sub:SOA} is confirmed and intensified, with even uPWS showing a similar behavior for $k = 32$ and $64$, corroborating the suspicion
that the NN is more sensitive to variations in its convolutional layers. On the other hand, uCWS and uECSQ tend to perform better than the other two techniques, exhibiting higher robustness for small values of $k$. The results for regression are more stable, as here the tendency is similar to that observed on FC layers: uUQ and uPWS are again competitive with the other two methods, with uUQ achieving the lowest MSE on KIBA and even when using the smallest $k$ tested. In general, performance is slightly worse and more unstable than that obtained by pruning the same layers, mainly for lower values of $k$. This suggests to apply not too radical quantizations, and if possible using pruning instead of quantization on convolutional layers. However, low pruning percentages are not enough to fully exploit \HAC\ and \sHAC, as well as CSC; thus, for compressing these layers in the experiments of next section, we decided to apply quantization without pruning (to not amplify the above-mentioned instability) and to employ index map representation to compress the obtained tensors. In fact, with low pruning rates, this method showed the best trade-off between compression and dot product time (Sect.~\ref{sub:soa}).
}
\subsection{Compressing dense and convolutional layers}\label{sub:final}
%
{In this section we present the more general experiment, aiming at compressing and storing in space-conscious formats all layers of a NN.
The results of
Sect.~\ref{cnn-compr}
suggested to apply low levels of pruning and quantization to convolutional layers,
making the index map more suitable to store such layers. In this setting, it is more convenient to not apply pruning, since it helps preventing further accuracy decay, and it does not reduce the
occupancy ratio. Therefore, an hybrid format is used in this experiment: index map for convolutional layers, and \HAC\ or \sHAC\ for FC layers, which undergo both pruning and quantization.
}
{It is worth nothing that extending the quantization via WS to all layers means, in the unified setup, that convolutional and FC layers will share the same representatives;  this experiment offers thereby also a full insight  of  how  the  proposed  strategies  interact  and perform  when  globally applied to the network.
In accordance to the results obtained in Sects~\ref{subsub:icpr} and~\ref{subsub:quantCNN}, the values $k\in \{32, 64, 128, 256\}$ have been tested.
On the other side, to prune FC layers, the results obtained in Sect.~\ref{subsub:icpr} suggested to  adopt different values for $p$ on each dataset: $p \in \{90, 92, 95, 97, 99\}$ for MNIST and CIFAR-10, $p \in \{50, 55, 60, 65, 70\}$ for KIBA, and $p \in \{70, 75, 80, 85, 90\}$ for DAVIS.

Finally, the matrices of FC layers have been stored with \HAC{} when it was more convenient than \sHAC\, and this is marked in Supplementary Tables
S8--S11
using $*$ besides the corresponding values. To better understand these results, we remark that occupancy ratios reported in the previous sections where limited to FC layers, not to the whole net; hence, since for the index map method  the occupancy ratio cannot be lower than $0.25$ (even larger when $k > 256$), the overall occupancy ratio is destined to increase with reference to that of FC layers. As a confirmation, the lowest occupancy registered is $0.0486$ (uPWS, DAVIS, $k$=32). Further, on VGG19 (the model with larger convolutional block) occupancy ratios are in average higher than on DeepDTA.
A first important observation is that performance tends to deteriorate in this setting (as expected), mainly on classification, compared with those obtained when compressing solely convolutional or FC layers. Nevertheless, the top performance achieved is still higher than the baseline on MNIST (uCWS, $p=92$, $k=256$) and occupying only the $16.666\%$ of the original network, whereas on CIFAR-10 data the best accuracy is only the $0.14\%$ smaller than the baseline (uUQ, $p=95$, $k=256$), while occupying just the $14.2\%$. Besides, the occupancy on MNIST can be halved to the detriment of $0.19\%$ of accuracy (uPWS, $p=97$, $k=32$).
The results with regression are more stable, confirming the trend shown in previous sections.
On KIBA, the top performance improves the baseline by $6\%$ with an occupancy of $0.1320$ (uCWS, $p=50$, $k=64$), and it can compress up to $0.0806$, still outperforming the uncompressed model (uUQ, $p=60$, $k=32$). On DAVIS, most configurations improve the baseline MSE, among which the best occupancy is $0.0544$ (uUQ, $p=90$, $k=32$), which means around $20\times$ smaller than the original NN. These results are quite impressive if we consider that model structure has not been modified.
Concerning quantization, the results support most of the analyses
done so far,
with uPWS and uUQ sometimes showing unstable performance for low values of $k$ on classification, uUQ (and sometimes uPWS) often exhibiting the best compression, and uCWS or uECSQ being good compromises between accuracy and space.

Summing up, on all models and data our format can reduce more than $5\times$ the pre-trained NN size, with peaks of around $20\times$, often improving or nearly matching its performance. The higher variability of results observed on some configurations likely depends on the fact that the actual number of clusters $k$ can be smaller than the reported one, due to a potential centroid overlapping during retraining (see Sect.~\ref{sub:WSquant}).
}

\section{Conclusions}
This work proposed two \emph{structure-preserving} CNN representations, \HAC\ and \sHAC, to compress pre-trained models while substantially maintaining their accuracy. As they exploit weight pruning and quantization, an extended comparison of four quantization methods based on WS, namely CWS, PWS, UQ and ECSQ, has been initially performed, suggesting that UQ is in average preferable when enough $k$ distinct weights are available (that is,  quantization is not too drastic), whereas EQCS should be used in the remaining cases, as an alternative to PWS on classification and to CWS on regression problems. Further, on convolutional layers CWS and ECQS are in general more stable with respect to $k$ variations on classification problems. \HAC\ and \sHAC\ exploited the quantized and pruned weight matrices to apply Huffman coding on the resulting source, and storing the resulting bitstring in a compact array format. A dedicated dot procedure has been implemented to supply an efficient matrix multiplication to evaluate the NN. In a comparison with state-of-the-art matrix compression methods, our formats have definitely shown higher compression rates when matrices are quantized and sufficiently sparse, and in some cases even a faster dot product. When applied to four benchmarks and two publicly retrievable pre-trained models, in most cases they preserved or improved performance,
while reducing memory requirements up to around $20\times$. The computed occupancy resulted in all cases much smaller than the theoretical upper bound derived here, thus suggesting that more strict bounds can probably be deduced. We found compressing convolutional layers being more critical than the dense ones, which is reasonable considering the different role they have in the input processing, but preventing to obtain high levels of pruning and quantization. This limited the application of our formats to these layers, thus future studies will be conveyed in this direction, to extend the representation so as to better exploit  lower pruning and quantization levels. Moreover, a finer level of parallelism can be introduced in the dot procedure to further speed up computations, by considering the vector of the offsets in the bitstring denoting the beginning of each column of the compressed matrix, and partitioning the columns in chunks assigned to different threads/cores to perform the corresponding  dot products.
 Finally,  coding methodologies less sensitive to source statistics, known as \textit{universal lossless source coding} (e.g., the Lempel–Ziv source coding), can be applied to reduce memory requirements, since they exhibit a smaller overhead  than Huffman coding.

\section*{Acknowledgements}
This work has been supported by the Italian MUR PRIN project “Multicriteria data structures and algorithms: from compressed to learned
indexes, and beyond” (Prot. 2017WR7SHH).

\bibliographystyle{IEEEtran}
\bibliography{references.bib}

\end{document}


\title{Supplementary materials for the article: "Compact representations of convolutional neural networks via weight pruning and quantization"}

\author{\IEEEauthorblockN{Giosuè Cataldo Marinò, Alessandro Petrini, Dario Malchiodi, Marco Frasca}\\
\IEEEauthorblockA{\textit{Università degli Studi di Milano},
Milano, Italy}\protect\\
E-mail: giosue.marino@studenti.unimi.it, \{alessandro.petrini, dario.malchiodi, marco.frasca\}@unimi.it}

\markboth{Journal of \LaTeX\ Class Files,~Vol.~14, No.~8, August~2015}%
{Shell \MakeLowercase{\textit{et al.}}: Bare Demo of IEEEtran.cls for IEEE Journals}

\maketitle


\begin{table*}[ht]
\centering
\caption{Top testing performance achieved by compression techniques. \textit{Type} is the compression technique, while \textit{Perf} contains Accuracy for VGG19 and MSE for DeepDTA (baseline in brackets). \textit{$\psi$} is the occupancy ratio, whereas $*$ denotes \sHAC\ representation as the lowest occupancy on that setting (w.r.t. \HAC). In bold the best results on each couple Net-Dataset.}
\label{best_perf_compr}
\begin{tabular}{|l|l|c|rr|}
\hline
\rowcolor[HTML]{B6D7A8}
\multicolumn{1}{|c|}{\cellcolor[HTML]{B6D7A8}\textbf{Net-Dataset}} & \multicolumn{1}{c|}{\cellcolor[HTML]{B6D7A8}\textbf{Type}} & \textbf{Config} & \multicolumn{1}{c|}{\cellcolor[HTML]{9FC5E8}\textbf{Perf}} & \multicolumn{1}{c|}{\cellcolor[HTML]{9FC5E8}\textbf{$\psi$}} \\ \hline
 & Pr & 96 & 0.9954 & 0.0800 \\ \cline{2-2}
 & CWS & 128-32-32 & 0.9957 & 0.3210 \\ \cline{2-2}
 & PWS & 32-32-2 & \textbf{0.9958} & 0.3090 \\ \cline{2-2}
 & Pr/CWS-a & 96/128-32-32 & 0.9956 & 0.0390* \\ \cline{2-2}
 & Pr/CWS-b & 96/128-32-32 & 0.9956 & 0.0390* \\ \cline{2-2}
 & Pr/PWS-a & 96/32-128-32 & 0.9956 & \textbf{0.0260*} \\ \cline{2-2}
\multirow{-7}{*}{\begin{tabular}[c]{@{}c@{}}VGG19-MNIST\\(0.9954)\end{tabular}} & Pr/PWS-b & 50/32-32-2 & \textbf{0.9958} & 0.1870 \\ \hline
 & Pr & 60 & 0.9365 & 0.8000 \\ \cline{2-2}
 & CWS & 32-32-2 & \textbf{0.9371} & 0.3060 \\ \cline{2-2}
 & PWS & 32-2-32 & 0.9363 & 0.0910 \\ \cline{2-2}
 & Pr/CWS-a & 60/2-2-32 & 0.9366 & 0.0880 \\ \cline{2-2}
 & Pr/CWS-b & 50/32-32-2 & 0.9370 & 0.2160 \\ \cline{2-2}
 & Pr/PWS-a & 60/2-2-32 & 0.9363 & 0.0880 \\ \cline{2-2}
\multirow{-7}{*}{\begin{tabular}[c]{@{}c@{}}VGG19-CIFAR10\\ (0.9344)\end{tabular}} & Pr/PWS-b & 98/32-2-32 & 0.9365 & \textbf{0.0120*} \\ \hline
 & Pr & 60 & \textbf{0.1599} & 0.8000 \\ \cline{2-2}
 & CWS & 128-128-32-2 & 0.1679 & 0.3900 \\ \cline{2-2}
 & PWS & 32-128-128-32 & 0.1761 & 0.4250 \\ \cline{2-2}
 & Pr/CWS-a & 60/32-128-2-32 & 0.1666 & \textbf{0.1870} \\ \cline{2-2}
 & Pr/CWS-b & 30/128-128-32-2 & 0.1644 & 0.3300 \\ \cline{2-2}
 & Pr/PWS-a & 60/128-128-128-32 & 0.1769 & 0.2070 \\ \cline{2-2}
\multirow{-7}{*}{\begin{tabular}[c]{@{}c@{}}DeepDTA-KIBA\\ (0.1756)\end{tabular}} & Pr/PWS-b & 40/32-128-128-32 & 0.1683 & 0.2910 \\ \hline
 & Pr & 80 & \textbf{0.2242} & 0.4000 \\ \cline{2-2}
 & CWS & 128-2-128-2 & 0.2320 & 0.2120 \\ \cline{2-2}
 & PWS & 128-32-32-32 & 0.2430 & 0.3240 \\ \cline{2-2}
 & Pr/CWS-a & 80/32-128-2-32 & 0.2341 & \textbf{0.1050} \\ \cline{2-2}
 & Pr/CWS-b & 40/128-2-128-2 & 0.2826 & 0.1910 \\ \cline{2-2}
 & Pr/PWS-a & 80/128-128-32 & 0.2302 & 0.1220 \\ \cline{2-2}
\multirow{-7}{*}{\begin{tabular}[c]{@{}c@{}}DeepDTA-DAVIS\\ (0.3223)\end{tabular}} & Pr/PWS-b & 60/128-32-32-32 & 0.2353 & 0.1600 \\ \hline
\end{tabular}
\end{table*}

\begin{table*}[ht]
\centering
\caption{Best occupancy ratio ensuring no decay in performance w.r.t.\ the uncompressed model. Same notations as in Table~\ref{best_perf_compr}.
}\label{best_space_compr}
\begin{tabular}{|l|l|c|rr|}
\hline
\rowcolor[HTML]{B6D7A8}
\multicolumn{1}{|c|}{\cellcolor[HTML]{B6D7A8}\textbf{Net-Dataset}} & \multicolumn{1}{c|}{\cellcolor[HTML]{B6D7A8}\textbf{Type}} & \textbf{Config} & \multicolumn{1}{c|}{\cellcolor[HTML]{9FC5E8}\textbf{Perf}} & \multicolumn{1}{c|}{\cellcolor[HTML]{9FC5E8}\textbf{Psi}} \\ \hline
 & Pr & 97 & 0.9953 & 0.0600 \\ \cline{2-2}
 & CWS & 128-2-32 & 0.9954 & 0.1040 \\ \cline{2-2}
 & PWS & 1024-2-32 & 0.9955 & 0.1260 \\ \cline{2-2}
 & Pr/CWS-a & 96/2-128-2 & 0.9954 & 0.0380* \\ \cline{2-2}
 & Pr/CWS-b & 96/128-32-32 & \textbf{0.9956} & 0.0390* \\ \cline{2-2}
 & Pr/PWS-a & 96/32-128-32 & \textbf{0.9956} & 0.0260* \\ \cline{2-2}
 \multirow{-7}{*}{\begin{tabular}[c]{@{}c@{}}VGG19-MNIST\\(0.9954)\end{tabular}} & Pr/PWS-b & 97/32-32-2 & 0.9955 & \textbf{0.0180*} \\
 \hline
 & Pr & 99 & 0.9357 & 0.0200 \\ \cline{2-2}
 & CWS & 2-2-32 & 0.9360 & 0.0630 \\ \cline{2-2}
 & PWS & 2-2-32 & 0.9351 & 0.0630 \\ \cline{2-2}
 & Pr/CWS-a & 60/2-2-32 & \textbf{0.9366}& 0.0880 \\ \cline{2-2}
 & Pr/CWS-b & 99/32-32-2 & 0.9358 & \textbf{0.0060*} \\ \cline{2-2}
 & Pr/PWS-a & 60/2-2-32 & 0.9363 & 0.0880 \\ \cline{2-2}
\multirow{-7}{*}{\begin{tabular}[c]{@{}c@{}}VGG19-CIFAR10\\ (0.9344)\end{tabular}} & Pr/PWS-b & 99/32-2-32 & 0.9363 & \textbf{0.0060*} \\ \hline
 & Pr & 60 & \textbf{0.1599} & 0.8000 \\ \cline{2-2}
 & CWS & 32-32-2-2 & 0.1723 & 0.2280 \\ \cline{2-2}
 & PWS & 32-128-128-32 & 0.1761 & 0.4250 \\ \cline{2-2}
 & Pr/CWS-a & 60/32-2-32-2 & 0.1739 & \textbf{0.1270} \\ \cline{2-2}
 & Pr/CWS-b & 60/128-128-32-2 & 0.1712 & 0.2220 \\ \cline{2-2}
 & Pr/PWS-a & 60/128-128-128-32 & 0.1769 & 0.2070 \\ \cline{2-2}
\multirow{-7}{*}{\begin{tabular}[c]{@{}c@{}}DeepDTA-KIBA\\ (0.1756)\end{tabular}} & Pr/PWS-b & 50/32-128-128-32 & 0.1702 & 0.2430 \\ \hline
 & Pr & 90 & 0.2425 & 0.2000 \\ \cline{2-2}
 & CWS & 2-2-2-2 & 0.2840 & 0.0630 \\ \cline{2-2}
 & PWS & 32-32-2-32 & 0.2567 & 0.2370 \\ \cline{2-2}
 & Pr/CWS-a & 80/32-2-2-32 & \textbf{0.2367} & 0.0790 \\ \cline{2-2}
 & Pr/CWS-b & 60/128-2-128-2 & 0.2906 & 0.1480 \\ \cline{2-2}
 & Pr/PWS-a & 90/128-32-32-32 & 0.2671 & \textbf{0.0600*} \\ \cline{2-2}
\multirow{-7}{*}{\begin{tabular}[c]{@{}c@{}}DeepDTA-DAVIS\\ (0.3223)\end{tabular}} & Pr/PWS-b & 80/32-2-2-32 & 0.2943 & 0.0770 \\ \hline
\end{tabular}
\end{table*}

\begin{figure*}[!h]
\centering
\scriptsize
\vspace{-0.5cm}
\begin{tabular}{ccc}
\vspace{-0.7cm}
\hspace{-0.5cm}\includegraphics[width=0.3\textwidth]{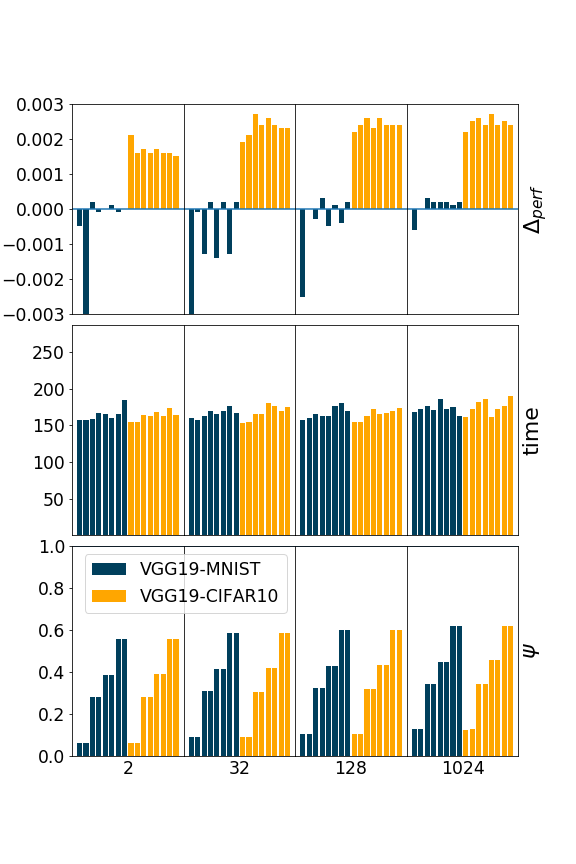}	&
\hspace{-0.5cm}\includegraphics[width=0.3\textwidth]{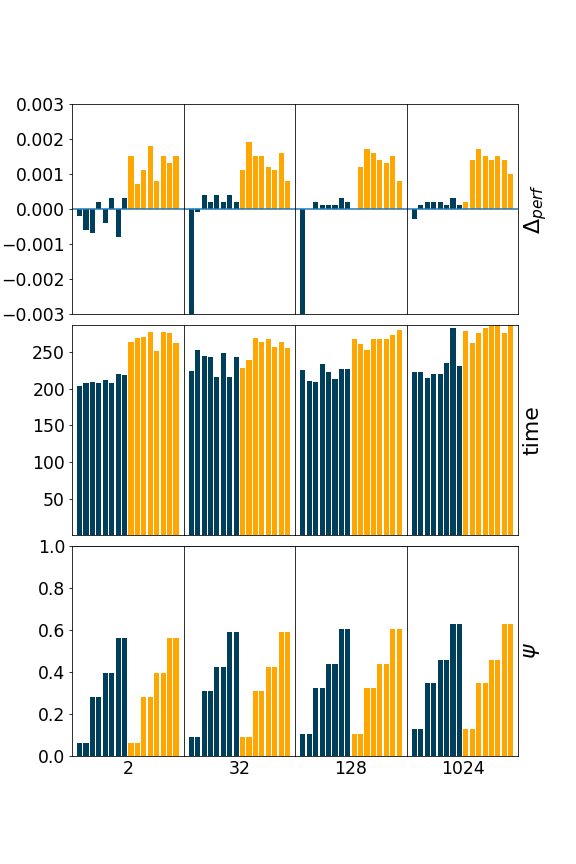} &
\hspace{-0.5cm}\includegraphics[width=0.3\textwidth]{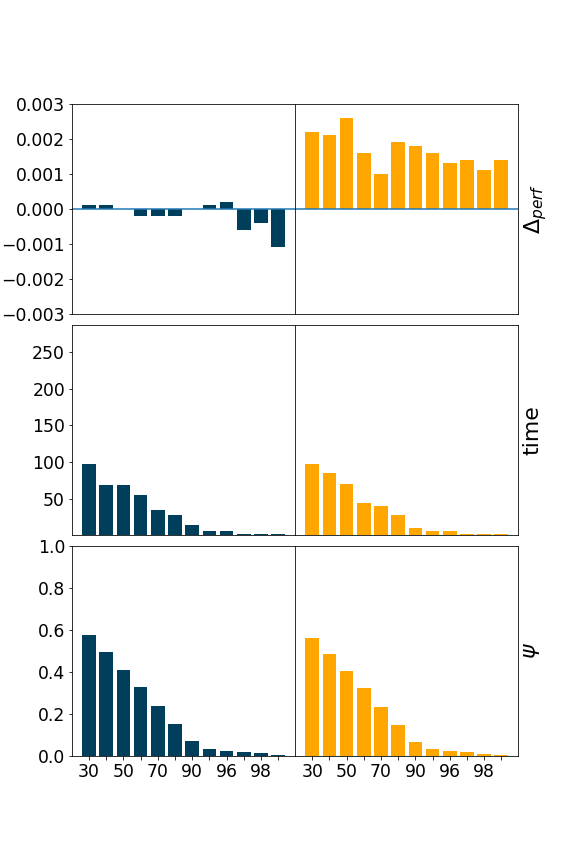}\\[-3pt]
\hspace{-0.5cm}(a) & \hspace{-0.5cm} (b) & \hspace{-0.5cm}(c)\\[-2pt]\vspace{-0.7cm}
\hspace{-0.5cm}\includegraphics[width=0.295\textwidth]{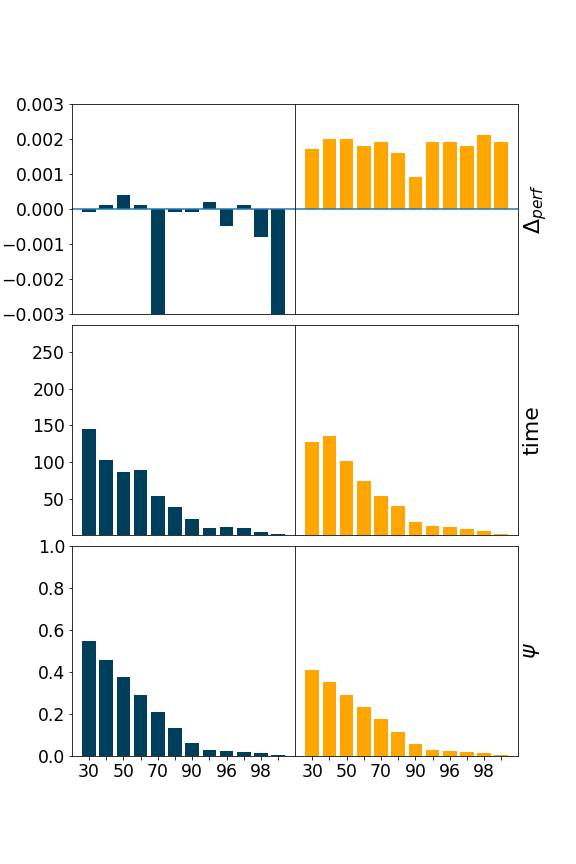}	&
\hspace{-0.5cm}\includegraphics[height=8.02cm, width=0.3\textwidth]{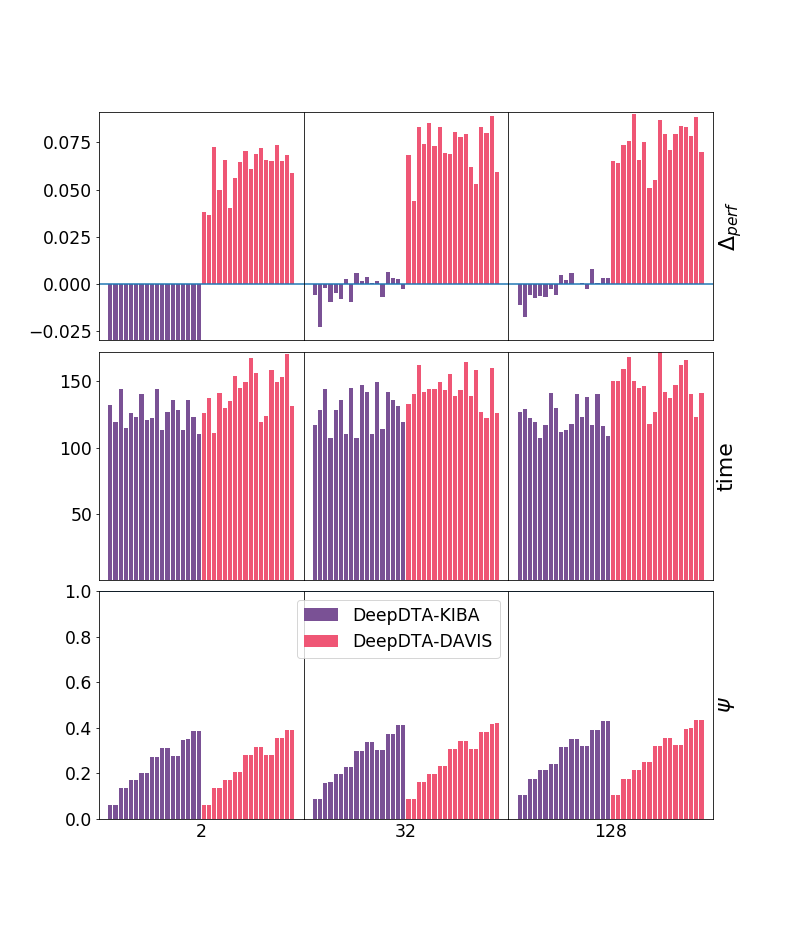} &
\hspace{-0.5cm}\includegraphics[height=8.02cm,width=0.3\textwidth]{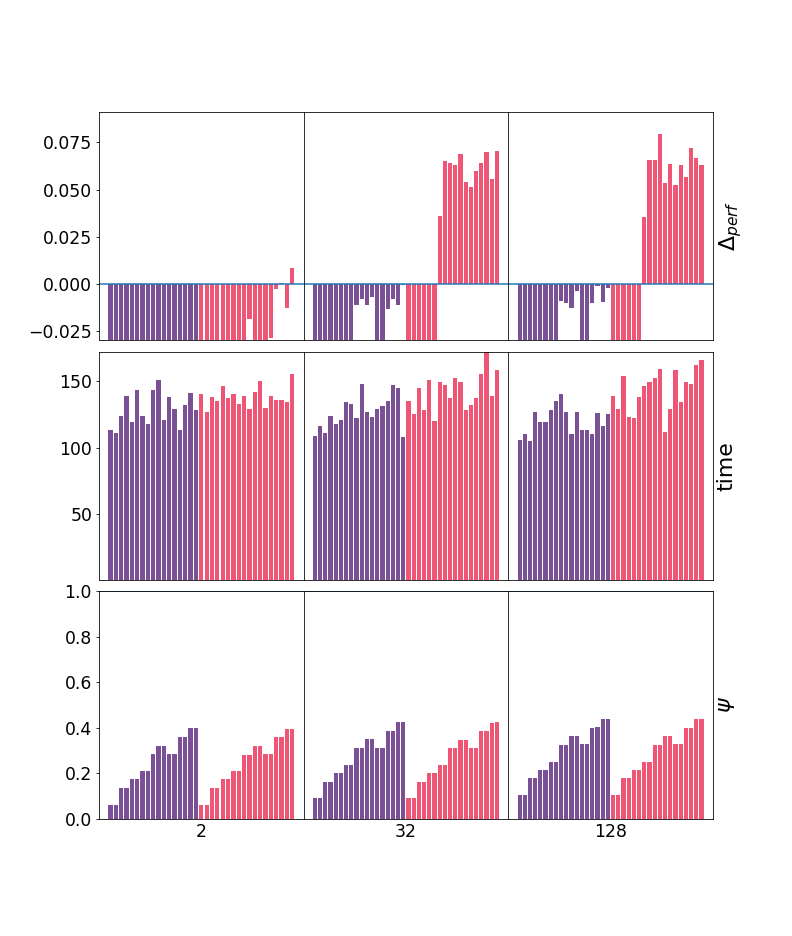}\\[-3pt]
\hspace{-0.5cm}(d) & \hspace{-0.5cm} (e) & \hspace{-0.5cm}(f) \\[-2pt]\vspace{-0.7cm}
\hspace{-0.5cm}\includegraphics[width=0.3\textwidth]{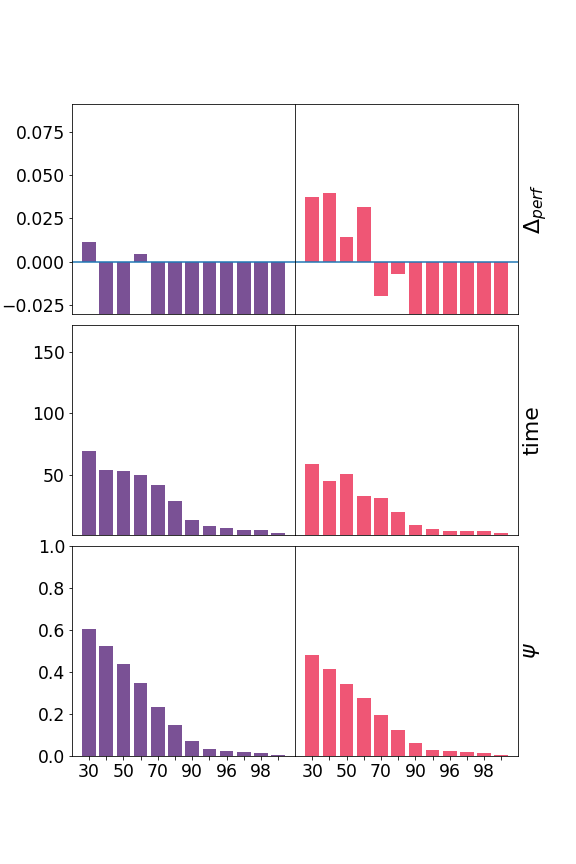}	&
\hspace{-0.5cm}\includegraphics[width=0.3\textwidth]{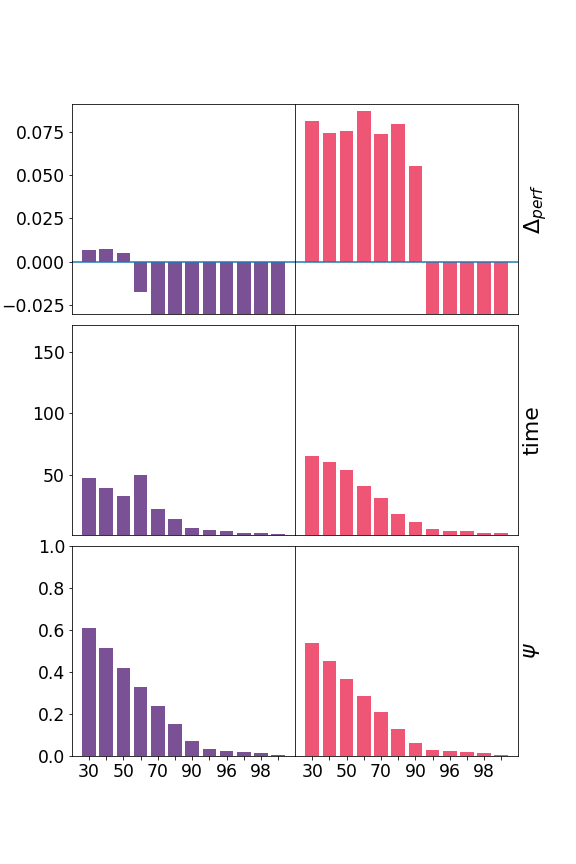} & \\
\hspace{-0.5cm}(g) & \hspace{-0.5cm} (h) &
\end{tabular}
\caption{Overall testing performance for compression methods: a) CWS (\HAC\ format), b) PWS (\HAC), c) Pr-CWS a (\sHAC), d) Pr-PWS a (\sHAC), e), f), g) and h) the same a), b), c) and  d) methods, but on regression datasets. X-axis report the compression parameters, while Y-axis measures the performance in term of space occupancy, time ratio, and testing performance (from top to bottom).}\label{fig:all_etrics}
\normalsize
\end{figure*}

\begin{figure*}[!th]
\centering
\scriptsize
\vspace{-0.5cm}
\includegraphics[width=1.0\textwidth]{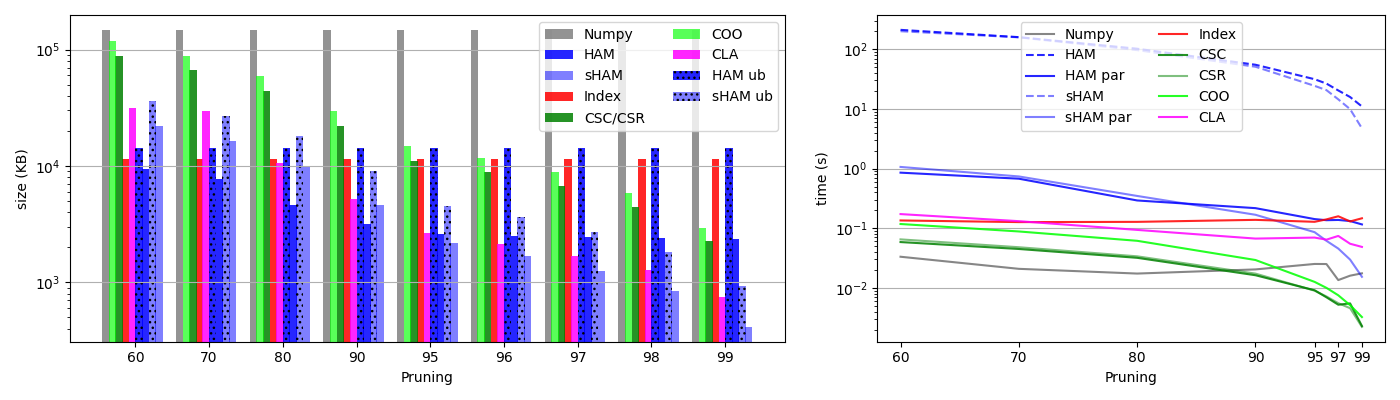}
\vspace{0.5pt}
\includegraphics[width=1.0\textwidth]{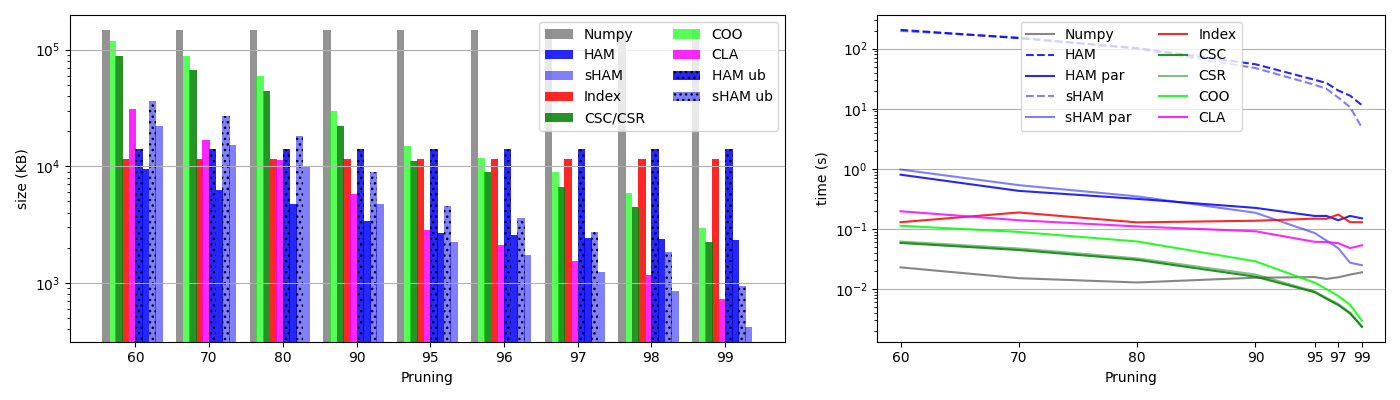}
\caption{Average execution times in seconds (graphs on the right) for performing $8$ vector matrix product, and memory footprint in KiloBytes (graphs on the left) for storing the compressed representation of the three fully connected layer weight matrices of the VGG19 model. Top rows refer to the model trained on CIFAR10 dataset, bottom row on MNIST dataset. All matrices were subjected to different degrees of pruning and were quantized by weight sharing (CWS) with 32 values. Dotted bars represent the upper bound of \HAC{} and \sHAC{} sizes evaluated with Corollaries 1 and 2.  Times are reported in logarithmic scale.}\label{fig:cifarMnistTimeSpace256}
\normalsize
\end{figure*}

\begin{table*}[ht]
\centering
\caption{Comparison between unified quantization (one $k$ for all layers) and non-unified quantization (different $k$ for each layer). The chosen configurations for non-unified methods are those reported in Table \ref{best_space_compr}. Models marked with * are those where \sHAC\ compressed more. Testing performance baseline in brackets.}
\label{tab:comparision-unique-vs-2}
\resizebox{1 \columnwidth}{!}{
\begin{tabular}{|l|l|c|rr|}
\hline
\rowcolor[HTML]{B6D7A8}
\multicolumn{1}{|c|}{\cellcolor[HTML]{B6D7A8}\textbf{Net-Dataset}} & \multicolumn{1}{c|}{\cellcolor[HTML]{B6D7A8}\textbf{Type}} & \textbf{Config} & \multicolumn{1}{c|}{\cellcolor[HTML]{9FC5E8}\textbf{Perf}} & \multicolumn{1}{c|}{\cellcolor[HTML]{9FC5E8}\textbf{$\psi$}} \\ \hline
 & CWS & 128-2-32 & 0.9954 & 0.1040 \\ \cline{2-2}
 & uCWS & 162 & 0.9957 & 0.2265 \\ \cline{2-2}
 & uCWS* & 16 & 0.9954 & 0.1215 \\ \cline{2-2}
 & PWS & 1024-2-32 & 0.9955 & 0.1260 \\ \cline{2-2}
 & uPWS & 1058 & 0.9955 & 0.3121 \\ \cline{2-2}
\multirow{-6}{*}{\begin{tabular}[c]{@{}c@{}}VGG19-MNIST\\(0.9954)\end{tabular}} & uPWS* & 2 & 0.9955 & 0.0313 \\ \hline
 & CWS & 2-2-32 & 0.9360 & 0.0630 \\ \cline{2-2}
 & uCWS & 36 & 0.9367 & 0.1567 \\ \cline{2-2}
 & uCWS* & 2 & 0.9362 & 0.0313 \\ \cline{2-2}
 & PWS & 2-2-32 & 0.9351 & 0.0630 \\ \cline{2-2}
 & uPWS & 36 & 0.9364 & 0.1576 \\ \cline{2-2}
\multirow{-6}{*}{\begin{tabular}[c]{@{}c@{}}VGG19-CIFAR10\\ (0.9344)\end{tabular}} & uPWS* & 2 & 0.9364 & 0.0313 \\ \hline
 & CWS & 32-32-2-2 & 0.1723 & 0.2280 \\ \cline{2-2}
 & uCWS & 68 & 0.1601 & 0.1843 \\ \cline{2-2}
 & uCWS* & 8 & 0.1661 & 0.0848 \\ \cline{2-2}
 & PWS & 32-128-128-32 & 0.1761 & 0.4250 \\ \cline{2-2}
 & uPWS & 320 & 0.1631 & 0.2642 \\ \cline{2-2}
\multirow{-6}{*}{\begin{tabular}[c]{@{}c@{}}DeepDTA-KIBA\\ (0.1756)\end{tabular}} & uPWS* & 32 & 0.1718 & 0.1553 \\ \hline
 & CWS & 2-2-2-2 & 0.2840 & 0.0630 \\ \cline{2-2}
 & uCWS & 8 & 0.2637 & 0.0900 \\ \cline{2-2}
 & uCWS* & 4 & 0.2751 & 0.0606 \\ \cline{2-2}
 & PWS & 32-32-2-32 & 0.2430 & 0.2370 \\ \cline{2-2}
 & uPWS & 98 & 0.2276 & 0.2091 \\ \cline{2-2}
\multirow{-6}{*}{\begin{tabular}[c]{@{}c@{}}DeepDTA-DAVIS\\ (0.3223)\end{tabular}} & uPWS* & 8 & 0.2783 & 0.0865 \\ \hline
\end{tabular}
}
\end{table*}

\begin{table*}[]
\centering
\caption{Summary of the perfomance of DeepDTA DNN trained and tested over the KIBA and DAVIS datasets, after applying different quantization techniques on the Dense layers. Perf. measured in MSE. Values in the $\psi$ column report the ratio between the \HAC{} compressed and uncompressed network sizes. In brackets the baseline MSE values.}
\label{tab:summary-perf-DTA}
\begin{tabular}{cl|rr|rr|}
\cline{3-6}
\multicolumn{1}{l}{} &  & \multicolumn{2}{c|}{\cellcolor[HTML]{B4A7D6}\textbf{\begin{tabular}[c]{@{}c@{}}KIBA\\ (0.1756)\end{tabular}}} & \multicolumn{2}{c|}{\cellcolor[HTML]{B4A7D6}\textbf{\begin{tabular}[c]{@{}c@{}}DAVIS\\ (0.3223)\end{tabular}}} \\ \hline
\rowcolor[HTML]{9FC5E8}
\multicolumn{1}{|c|}{\cellcolor[HTML]{B6D7A8}\textbf{$k$}} & \multicolumn{1}{c|}{\cellcolor[HTML]{B6D7A8}\textbf{Method}} & \multicolumn{1}{c|}{\cellcolor[HTML]{9FC5E8}\textbf{Perf}} & \multicolumn{1}{c|}{\cellcolor[HTML]{9FC5E8}\textbf{$\psi$}} & \multicolumn{1}{c|}{\cellcolor[HTML]{9FC5E8}\textbf{Perf}} & \multicolumn{1}{c|}{\cellcolor[HTML]{9FC5E8}\textbf{$\psi$}} \\ \hline
\multicolumn{1}{|c|}{\cellcolor[HTML]{FFFFFF}} & uCWS & 80.0142 & 0.0313 & 22.3118 & 0.0313 \\ \cline{2-2}
\multicolumn{1}{|c|}{\cellcolor[HTML]{FFFFFF}} & uPWS & 0.7699 & 0.0313 & 29.6357 & 0.0313 \\ \cline{2-2}
\multicolumn{1}{|c|}{\cellcolor[HTML]{FFFFFF}} & uUQ & 105.6028 & 0.0313 & 22.3100 & 0.0313 \\ \cline{2-2}
\multicolumn{1}{|c|}{\multirow{-4}{*}{\cellcolor[HTML]{FFFFFF}2}} & uECSQ & 80.0640 & 0.0313 & 22.3111 & 0.0313 \\ \hline
\multicolumn{1}{|c|}{} & uCWS & 0.1625 & 0.1197 & 0.2321 & 0.1211 \\ \cline{2-2}
\multicolumn{1}{|c|}{} & uPWS & 0.1847 & 0.1225 & 0.2581 & 0.1217 \\ \cline{2-2}
\multicolumn{1}{|c|}{} & uUQ & 0.1696 & 0.0523 & 22.3108 & 0.0413 \\ \cline{2-2}
\multicolumn{1}{|c|}{\multirow{-4}{*}{16}} & uECSQ & 0.1631 & 0.1031 & 0.2309 & 0.1059 \\ \hline
\multicolumn{1}{|c|}{} & uCWS & 0.1606 & 0.1491 & 0.2203 & 0.1502 \\ \cline{2-2}
\multicolumn{1}{|c|}{} & uPWS & 0.1713 & 0.1551 & 0.2432 & 0.1550 \\ \cline{2-2}
\multicolumn{1}{|c|}{} & uUQ & 0.1639 & 0.0841 & 6.9812 & 0.0652 \\ \cline{2-2}
\multicolumn{1}{|c|}{\multirow{-4}{*}{32}} & uECSQ & 0.1600 & 0.0933 & 0.2283 & 0.1183 \\ \hline
\multicolumn{1}{|c|}{} & uCWS & 0.1608 & 0.1825 & 0.2237 & 0.1853 \\ \cline{2-2}
\multicolumn{1}{|c|}{} & uPWS & 0.1674 & 0.1873 & 0.2258 & 0.1873 \\ \cline{2-2}
\multicolumn{1}{|c|}{} & uUQ & 0.1639 & 0.1161 & 0.2244 & 0.0959 \\ \cline{2-2}
\multicolumn{1}{|c|}{\multirow{-4}{*}{64}} & uECSQ & 0.1618 & 0.1237 & 0.2251 & 0.1636 \\ \hline
\multicolumn{1}{|c|}{} & uCWS & 0.1608 & 0.2149 & 0.2214 & 0.2169 \\ \cline{2-2}
\multicolumn{1}{|c|}{} & uPWS & 0.1659 & 0.2194 & 0.2323 & 0.2194 \\ \cline{2-2}
\multicolumn{1}{|c|}{} & uUQ & 0.1620 & 0.1498 & 0.2304 & 0.1272 \\ \cline{2-2}
\multicolumn{1}{|c|}{\multirow{-4}{*}{128}} & uECSQ & 0.1615 & 0.1894 & 0.2242 & 0.1967 \\ \hline
\multicolumn{1}{|c|}{} & uCWS & 0.1600 & 0.2446 & 0.2272 & 0.2494 \\ \cline{2-2}
\multicolumn{1}{|c|}{} & uPWS & 0.1628 & 0.2519 & 0.2228 & 0.2519 \\ \cline{2-2}
\multicolumn{1}{|c|}{} & uUQ & 0.1613 & 0.1772 & 0.2204 & 0.1593 \\ \cline{2-2}
\multicolumn{1}{|c|}{\multirow{-4}{*}{256}} & uECSQ & 0.1603 & 0.2187 & 0.2309 & 0.2069 \\ \hline
\end{tabular}
\end{table*}

\begin{table*}[!t]
\centering
\caption{Summary of the pruning and quantization configurations yielding the best performing model in terms of Accuracy. 
Testing performance baseline in brackets.}
\label{tab:bestperf-pr-q}
\begin{tabular}{|c|l|c|rr|}
\hline
\rowcolor[HTML]{B4A7D6}
\rowcolor[HTML]{B6D7A8}
\textbf{Net-Dataset} & \multicolumn{1}{c|}{\cellcolor[HTML]{B6D7A8}\textbf{Type}} & \textbf{$p$-$k$} & \multicolumn{1}{c|}{\cellcolor[HTML]{9FC5E8}\textbf{Perf}} & \multicolumn{1}{c|}{\cellcolor[HTML]{9FC5E8}\textbf{$\psi$}} \\ \hline
 & PruCWS & 60-16 & 0.9955 & 0.0777 \\ \cline{2-2}
 & PruPWS & 60-32 & 0.9955 & 0.0807 \\ \cline{2-2}
 & PruUQ* & 99-32 & 0.9954 & 0.0057\\ \cline{2-2}
\multirow{-4}{*}{\begin{tabular}[c]{@{}c@{}}VGG19-MNIST\\(0.9954)\end{tabular}} & PruECSQ* & 95-16 & 0.9954 & 0.0276 \\ \hline
 & PruCWS & 95-32 & 0.9365 & 0.0338 \\ \cline{2-2}
 & PruPWS & 60-16 & 0.9366 & 0.0617 \\ \cline{2-2}
 & PruUQ* & 99-16 & 0.9363 & 0.00567 \\ \cline{2-2}
\multirow{-4}{*}{\begin{tabular}[c]{@{}c@{}}VGG19-CIFAR-10\\ (0.9344)\end{tabular}} & PruECSQ & 60-16 & 0.9365 & 0.07537 \\ \hline
 & PruCWS & 60-16 & 0.1645 & 0.07907 \\ \cline{2-2}
 & PruPWS & 60-64 & 0.1700 & 0.09017 \\ \cline{2-2}
 & PruUQ & 60-64 & 0.1630 & 0.08247 \\ \cline{2-2}
\multirow{-4}{*}{\begin{tabular}[c]{@{}c@{}}DeepDTA-KIBA\\ (0.1756)\end{tabular}} & PruECSQ & 60-32 & 0.1645 & 0.0918 \\ \hline
 & PruCWS & 60-16 & 0.2218 & 0.0795 \\ \cline{2-2}
 & PruPWS & 60-64 & 0.2269 & 0.0875 \\ \cline{2-2}
 & PruUQ & 60-64 & 0.2203 & 0.0675 \\ \cline{2-2}
\multirow{-4}{*}{\begin{tabular}[c]{@{}c@{}}DeepDTA-DAVIS\\ (0.3223)\end{tabular}} & PruECSQ & 60-32 & 0.2199 & 0.0887 \\ \hline
\end{tabular}
\end{table*}

\begin{table*}[ht]
\centering
\caption{Summary of the pruning and quantization configurations obtaining the best performing model in terms of compression ratio. Testing performance baseline in brackets.}
\label{tab:bestspace-pr-q}
\begin{tabular}{|c|l|c|rr|}
\hline
\rowcolor[HTML]{B4A7D6}
\rowcolor[HTML]{B6D7A8}
\textbf{Net-Dataset} & \multicolumn{1}{c|}{\cellcolor[HTML]{B6D7A8}\textbf{Type}} & \textbf{$p$-$k$} & \multicolumn{1}{c|}{\cellcolor[HTML]{9FC5E8}\textbf{Perf}} & \multicolumn{1}{c|}{\cellcolor[HTML]{9FC5E8}\textbf{$\psi$}} \\ \hline
 & PruCWS & 80-64 & 0.9954 & 0.0577 \\ \cline{2-2}
 & PruPWS & 60-32 & 0.9954 & 0.0807 \\ \cline{2-2}
 & PruUQ* & 99-32 & 0.9954 & 0.0057 \\ \cline{2-2}
\multirow{-4}{*}{\begin{tabular}[c]{@{}c@{}}VGG19-MNIST\\(0.9954)\end{tabular}} & PruECSQ* & 95-16 & 0.9954 & 0.0276 \\ \hline
 & PruCWS* & 99-32 & 0.9361 & 0.0055 \\ \cline{2-2}
 & PruPWS* & 99-32 & 0.9344 & 0.0055 \\ \cline{2-2}
 & PruUQ* & 99-16 & 0.9363 & 0.0056 \\ \cline{2-2}
\multirow{-4}{*}{\begin{tabular}[c]{@{}c@{}}VGG19-CIFAR10\\ (0.9344)\end{tabular}} & PruECSQ* & 99-32 & 0.9363 & 0.0055 \\ \hline
 & PruCWS & 60-16 & 0.1645 & 0.0790 \\ \cline{2-2}
 & PruPWS & 60-64 & 0.1700 & 0.0901 \\ \cline{2-2}
 & PruUQ & 60-16 & 0.1698 & 0.0596 \\ \cline{2-2}
\multirow{-4}{*}{\begin{tabular}[c]{@{}c@{}}DeepDTA-KIBA\\ (0.1756)\end{tabular}} & PruECSQ & 60-16 & 0.1660 & 0.0766 \\ \hline
 & PruCWS* & 95-64 & 0.3037 & 0.0268 \\ \cline{2-2}
 & PruPWS* & 97-32 & 0.3187 & 0.0159 \\ \cline{2-2}
 & PruUQ & 90-16 & 0.2383 & 0.0365 \\ \cline{2-2}
\multirow{-4}{*}{\begin{tabular}[c]{@{}c@{}}DeepDTA-DAVIS\\ (0.3223)\end{tabular}} & PruECSQ* & 95-32 & 0.3127 & 0.0271 \\ \hline
\end{tabular}
\end{table*}

\begin{table*}[]
\center
\caption{Performance achieved after applying quantization via weight sharing techniques only to CNN layers. For DeepDTA, performance is measured as MSE; for VGG19, as Accuracy (baseline in brackets).}
\label{tab:quantization-cnn}
\begin{tabular}{cl|rrrr|}
\cline{3-6}
\multicolumn{1}{l}{} &  & \multicolumn{2}{c|}{\cellcolor[HTML]{9FC5E8}\textbf{VGG19}} & \multicolumn{2}{c|}{\cellcolor[HTML]{9FC5E8}\textbf{DeepDTA}} \\ \hline
\rowcolor[HTML]{9FC5E8}
\multicolumn{1}{|c|}{\cellcolor[HTML]{B6D7A8}\textbf{$k$}} & \multicolumn{1}{c|}{\cellcolor[HTML]{B6D7A8}\textbf{Method}} & \multicolumn{1}{c|}{\cellcolor[HTML]{9FC5E8}{\begin{tabular}[c]{@{}c@{}}\textbf{MNIST}\\ \textbf{(0.9954)}\end{tabular}}} & \multicolumn{1}{c|}{\cellcolor[HTML]{9FC5E8}{\begin{tabular}[c]{@{}c@{}}\textbf{CIFAR-10}\\ \textbf{(0.9344)}\end{tabular}}} & \multicolumn{1}{c|}{\cellcolor[HTML]{9FC5E8}{\begin{tabular}[c]{@{}c@{}}\textbf{KIBA}\\ \textbf{(0.1756)}\end{tabular}}} & \multicolumn{1}{c|}{\cellcolor[HTML]{9FC5E8}{\begin{tabular}[c]{@{}c@{}}\textbf{DAVIS}\\ \textbf{(0.3223)}\end{tabular}}} \\ \hline
\multicolumn{1}{|c|}{} & uCWS & 0.9941 & 0.9109 & 0.1570 & 0.2221 \\ \cline{2-2}
\multicolumn{1}{|c|}{} & uPWS & 0.9918 & 0.1129 & 0.1634 & 0.2209 \\ \cline{2-2}
\multicolumn{1}{|c|}{} & uUQ & 0.1434 & 0.1210 & 0.1557 & 0.2234 \\ \cline{2-2}
\multicolumn{1}{|c|}{\multirow{-4}{*}{32}} & uECSQ & 0.9943 & 0.8941 & 0.1582 & 0.2301 \\ \hline
\multicolumn{1}{|c|}{} & uCWS & 0.9943 & 0.9036 & 0.1575 & 0.2219 \\ \cline{2-2}
\multicolumn{1}{|c|}{} & uPWS & 0.9932 & 0.1234 & 0.1602 & 0.2234 \\ \cline{2-2}
\multicolumn{1}{|c|}{} & uUQ & 0.2059 & 0.1746 & 0.1571 & 0.2214 \\ \cline{2-2}
\multicolumn{1}{|c|}{\multirow{-4}{*}{64}} & uECSQ & 0.9933 & 0.9013 & 0.1586 & 0.2326 \\ \hline
\multicolumn{1}{|c|}{} & uCWS & 0.9943 & 0.9133 & 0.1563 & 0.2189 \\ \cline{2-2}
\multicolumn{1}{|c|}{} & uPWS & 0.9934 & 0.9020 & 0.1579 & 0.2211 \\ \cline{2-2}
\multicolumn{1}{|c|}{} & uUQ & 0.2076 & 0.1842 & 0.1568 & 0.2216 \\ \cline{2-2}
\multicolumn{1}{|c|}{\multirow{-4}{*}{128}} & uECSQ & 0.9952 & 0.9030 & 0.1582 & 0.2315 \\ \hline
\multicolumn{1}{|c|}{} & uCWS & 0.9946 & 0.9196 & 0.1563 & 0.2193 \\ \cline{2-2}
\multicolumn{1}{|c|}{} & uPWS & 0.9936 & 0.9065 & 0.1571 & 0.2191 \\ \cline{2-2}
\multicolumn{1}{|c|}{} & uUQ & 0.2106 & 0.9336 & 0.1576 & 0.2189 \\ \cline{2-2}
\multicolumn{1}{|c|}{\multirow{-4}{*}{256}} & uECSQ & 0.9955 & 0.9073 & 0.1581 & 0.2204 \\ \hline
\end{tabular}
\end{table*}

\begin{table*}[]
\centering
\caption{Summary of the results of the application of pruning and quantization applied to Dense layers and quantization on Convolutional layers of VGG19 trained on the MNIST dataset. Performance is reported as accuracy; compression ratio ($\psi$) in brackets. Baseline of the uncompressed model reported in bold. Models compressed with \sHAC are marked with a star (*), otherwise \HAC was used.}
\label{tab:pr-and-q-dense-and-cnn-mnist}
\begin{tabular}{|c|l|ccccc|}
\hline
\rowcolor[HTML]{B4A7D6}
\multicolumn{7}{|c|}{\cellcolor[HTML]{B4A7D6}\textbf{VGG19 - MNIST (0.9954)}} \\ \hline
\rowcolor[HTML]{9FC5E8}
\cellcolor[HTML]{B6D7A8} & \multicolumn{1}{c|}{\cellcolor[HTML]{B6D7A8}} & \multicolumn{5}{c|}{\cellcolor[HTML]{9FC5E8}\textbf{ }} \\ \cline{3-7}
\rowcolor[HTML]{9FC5E8}
\multirow{-2}{*}{\cellcolor[HTML]{B6D7A8}\textbf{$k$ (CNN and Dense)}} & \multicolumn{1}{c|}{\multirow{-2}{*}{\cellcolor[HTML]{B6D7A8}\textbf{Method}}} & \multicolumn{1}{c|}{\cellcolor[HTML]{9FC5E8}\textbf{PR Dense: 90}} & \multicolumn{1}{c|}{\cellcolor[HTML]{9FC5E8}\textbf{PR Dense: 92}} & \multicolumn{1}{c|}{\cellcolor[HTML]{9FC5E8}\textbf{PR Dense: 95}} & \multicolumn{1}{c|}{\cellcolor[HTML]{9FC5E8}\textbf{PR Dense: 97}} & \textbf{PR Dense: 99} \\ \hline
 & uCWS & \begin{tabular}[c]{@{}c@{}}0.9936\\ (0.1150)\end{tabular} & \begin{tabular}[c]{@{}c@{}}0.9937\\ (0.1144)\end{tabular} & \begin{tabular}[c]{@{}c@{}}0.9936\\ (0.1103*)\end{tabular} & \begin{tabular}[c]{@{}c@{}}0.9916\\ (0.1047*)\end{tabular} & \begin{tabular}[c]{@{}c@{}}0.9937\\ (0.0992*)\end{tabular} \\ \cline{2-7}
 & uPWS & \begin{tabular}[c]{@{}c@{}}0.1538\\ (0.0988)\end{tabular} & \begin{tabular}[c]{@{}c@{}}0.9935\\ (0.0981)\end{tabular} & \begin{tabular}[c]{@{}c@{}}0.9934\\ (0.0940*)\end{tabular} & \begin{tabular}[c]{@{}c@{}}0.9935\\ (0.0884*)\end{tabular} & \begin{tabular}[c]{@{}c@{}}0.9932\\ (0.0831*)\end{tabular} \\ \cline{2-7}
 & uUQ & \begin{tabular}[c]{@{}c@{}}0.2258\\ (0.1131)\end{tabular} & \begin{tabular}[c]{@{}c@{}}0.1374\\ (0.1128)\end{tabular} & \begin{tabular}[c]{@{}c@{}}0.1945\\ (0.1087*)\end{tabular} & \begin{tabular}[c]{@{}c@{}}0.1898\\ (0.1038*)\end{tabular} & \begin{tabular}[c]{@{}c@{}}0.1135\\ (0.0989*)\end{tabular} \\ \cline{2-7}
\multirow{-4}{*}{32} & uECSQ & \begin{tabular}[c]{@{}c@{}}0.9931\\ (0.1167)\end{tabular} & \begin{tabular}[c]{@{}c@{}}0.994\\ (0.1154)\end{tabular} & \begin{tabular}[c]{@{}c@{}}0.9945\\ (0.1107*)\end{tabular} & \begin{tabular}[c]{@{}c@{}}0.9944\\ (0.1048*)\end{tabular} & \begin{tabular}[c]{@{}c@{}}0.9929\\ (0.0992*)\end{tabular} \\ \hline
 & uCWS & \begin{tabular}[c]{@{}c@{}}0.9943\\ (0.1338)\end{tabular} & \begin{tabular}[c]{@{}c@{}}0.9941\\ (0.1323)\end{tabular} & \begin{tabular}[c]{@{}c@{}}0.9914\\ (0.1272*)\end{tabular} & \begin{tabular}[c]{@{}c@{}}0.9932\\ (0.1211*)\end{tabular} & \begin{tabular}[c]{@{}c@{}}0.9943\\ (0.1153*)\end{tabular} \\ \cline{2-7}
 & uPWS & \begin{tabular}[c]{@{}c@{}}0.9938\\ (0.1162)\end{tabular} & \begin{tabular}[c]{@{}c@{}}0.9942\\ (0.1150)\end{tabular} & \begin{tabular}[c]{@{}c@{}}0.9943\\ (0.1105*)\end{tabular} & \begin{tabular}[c]{@{}c@{}}0.9943\\ (0.1047*)\end{tabular} & \begin{tabular}[c]{@{}c@{}}0.9942\\ (0.0992*)\end{tabular} \\ \cline{2-7}
 & uUQ & \begin{tabular}[c]{@{}c@{}}0.1325\\ (0.1292)\end{tabular} & \begin{tabular}[c]{@{}c@{}}0.1334\\ (0.1289)\end{tabular} & \begin{tabular}[c]{@{}c@{}}0.1361\\ (0.1254*)\end{tabular} & \begin{tabular}[c]{@{}c@{}}0.1330\\ (0.1203*)\end{tabular} & \begin{tabular}[c]{@{}c@{}}0.1221\\ (0.1152*)\end{tabular} \\ \cline{2-7}
\multirow{-4}{*}{64} & uECSQ & \begin{tabular}[c]{@{}c@{}}0.9941\\ (0.1344)\end{tabular} & \begin{tabular}[c]{@{}c@{}}0.9942\\ (0.1328)\end{tabular} & \begin{tabular}[c]{@{}c@{}}0.9465\\ (0.1275*)\end{tabular} & \begin{tabular}[c]{@{}c@{}}0.9948\\ (0.1213*)\end{tabular} & \begin{tabular}[c]{@{}c@{}}0.9938\\ (0.1154*)\end{tabular} \\ \hline
 & uCWS & \begin{tabular}[c]{@{}c@{}}0.9944\\ (0.1509)\end{tabular} & \begin{tabular}[c]{@{}c@{}}0.9941\\ (0.1492)\end{tabular} & \begin{tabular}[c]{@{}c@{}}0.9952\\ (0.1438*)\end{tabular} & \begin{tabular}[c]{@{}c@{}}0.9931\\ (0.1377*)\end{tabular} & \begin{tabular}[c]{@{}c@{}}0.9941\\ (0.1316*)\end{tabular} \\ \cline{2-7}
 & uPWS & \begin{tabular}[c]{@{}c@{}}0.9937\\ (0.1336)\end{tabular} & \begin{tabular}[c]{@{}c@{}}0.9935\\ (0.1322)\end{tabular} & \begin{tabular}[c]{@{}c@{}}0.9939\\ (0.1271*)\end{tabular} & \begin{tabular}[c]{@{}c@{}}0.9796\\ (0.1211*)\end{tabular} & \begin{tabular}[c]{@{}c@{}}0.9592\\ (0.1153*)\end{tabular} \\ \cline{2-7}
 & uUQ & \begin{tabular}[c]{@{}c@{}}0.2090\\ (0.1302)\end{tabular} & \begin{tabular}[c]{@{}c@{}}0.1887\\ (0.1294)\end{tabular} & \begin{tabular}[c]{@{}c@{}}0.1661\\ (0.1257*)\end{tabular} & \begin{tabular}[c]{@{}c@{}}0.1601\\ (0.1204*)\end{tabular} & \begin{tabular}[c]{@{}c@{}}0.2007\\ (0.1152*)\end{tabular} \\ \cline{2-7}
\multirow{-4}{*}{128} & uECSQ & \begin{tabular}[c]{@{}c@{}}0.9946\\ (0.1520)\end{tabular} & \begin{tabular}[c]{@{}c@{}}0.9949\\ (0.1501)\end{tabular} & \begin{tabular}[c]{@{}c@{}}0.9925\\ (0.1442*)\end{tabular} & \begin{tabular}[c]{@{}c@{}}0.9926\\ (0.1378*)\end{tabular} & \begin{tabular}[c]{@{}c@{}}0.9943\\ (0.1316*)\end{tabular} \\ \hline
 & uCWS & \begin{tabular}[c]{@{}c@{}}0.9945\\ (0.1685)\end{tabular} & \begin{tabular}[c]{@{}c@{}}0.9955\\ (0.1666)\end{tabular} & \begin{tabular}[c]{@{}c@{}}0.9952\\ (0.1608*)\end{tabular} & \begin{tabular}[c]{@{}c@{}}0.9953\\ (0.1540*)\end{tabular} & \begin{tabular}[c]{@{}c@{}}0.9947\\ (0.1472*)\end{tabular} \\ \cline{2-7}
 & uPWS & \begin{tabular}[c]{@{}c@{}}0.9940\\ (0.1512)\end{tabular} & \begin{tabular}[c]{@{}c@{}}0.9939\\ (0.1494)\end{tabular} & \begin{tabular}[c]{@{}c@{}}0.9867\\ (0.1438*)\end{tabular} & \begin{tabular}[c]{@{}c@{}}0.9945\\ (0.1375*)\end{tabular} & \begin{tabular}[c]{@{}c@{}}0.9886\\ (0.1315*)\end{tabular} \\ \cline{2-7}
 & uUQ & \begin{tabular}[c]{@{}c@{}}0.1247\\ (0.1463)\end{tabular} & \begin{tabular}[c]{@{}c@{}}0.1170\\ (0.1459)\end{tabular} & \begin{tabular}[c]{@{}c@{}}0.1993\\ (0.1423*)\end{tabular} & \begin{tabular}[c]{@{}c@{}}0.1139\\ (0.1368*)\end{tabular} & \begin{tabular}[c]{@{}c@{}}0.2077\\ (0.1313*)\end{tabular} \\ \cline{2-7}
\multirow{-4}{*}{256} & uECSQ & \begin{tabular}[c]{@{}c@{}}0.9949\\ (0.1696)\end{tabular} & \begin{tabular}[c]{@{}c@{}}0.9952\\ (0.1674)\end{tabular} & \begin{tabular}[c]{@{}c@{}}0.9953\\ (0.1611*)\end{tabular} & \begin{tabular}[c]{@{}c@{}}0.9948\\ (0.1543*)\end{tabular} & \begin{tabular}[c]{@{}c@{}}0.9938\\ (0.1478*)\end{tabular} \\ \hline
\end{tabular}
\end{table*}

\begin{table*}[]
\centering
\caption{Summary of the results of the application of pruning and quantization applied to Dense layers and quantization on Convolutional layers of VGG19 trained on the CIFAR10 dataset. Performance is reported as accuracy; compression ratio ($\psi$) in brackets. Baseline of the uncompressed model reported in bold. Models compressed with \sHAC{} are marked with a star (*), otherwise \HAC{} was used.}
\label{tab:pr-and-q-dense-and-cnn-cifar}
\begin{tabular}{|c|l|ccccc|}
\hline
\rowcolor[HTML]{B4A7D6}
\multicolumn{7}{|c|}{\cellcolor[HTML]{B4A7D6}\textbf{VGG19 - CIFAR10 (0.9344)}} \\ \hline
\rowcolor[HTML]{9FC5E8}
\cellcolor[HTML]{B6D7A8} & \multicolumn{1}{c|}{\cellcolor[HTML]{B6D7A8}} & \multicolumn{5}{c|}{\cellcolor[HTML]{9FC5E8}\textbf{ }} \\ \cline{3-7}
\rowcolor[HTML]{9FC5E8}
\multirow{-2}{*}{\cellcolor[HTML]{B6D7A8}\textbf{$k$ (CNN and Dense)}} & \multicolumn{1}{c|}{\multirow{-2}{*}{\cellcolor[HTML]{B6D7A8}\textbf{Method}}} & \multicolumn{1}{c|}{\cellcolor[HTML]{9FC5E8}\textbf{PR Dense: 90}} & \multicolumn{1}{c|}{\cellcolor[HTML]{9FC5E8}\textbf{PR Dense: 92}} & \multicolumn{1}{c|}{\cellcolor[HTML]{9FC5E8}\textbf{PR Dense: 95}} & \multicolumn{1}{c|}{\cellcolor[HTML]{9FC5E8}\textbf{PR Dense: 97}} & \textbf{PR Dense: 99} \\ \hline
 & uCWS & \begin{tabular}[c]{@{}c@{}}0.8810\\ (0.1161)\end{tabular} & \begin{tabular}[c]{@{}c@{}}0.8818\\ (0.1150)\end{tabular} & \begin{tabular}[c]{@{}c@{}}0.8655\\ (0.1104*)\end{tabular} & \begin{tabular}[c]{@{}c@{}}0.8668\\ (0.1047*)\end{tabular} & \begin{tabular}[c]{@{}c@{}}0.8822\\ (0.0992*)\end{tabular} \\ \cline{2-7}
 & uPWS & \begin{tabular}[c]{@{}c@{}}0.1173\\ (0.0993)\end{tabular} & \begin{tabular}[c]{@{}c@{}}0.1170\\ (0.0983)\end{tabular} & \begin{tabular}[c]{@{}c@{}}0.1233\\ (0.0941*)\end{tabular} & \begin{tabular}[c]{@{}c@{}}0.1221\\ (0.0886*)\end{tabular} & \begin{tabular}[c]{@{}c@{}}0.1133\\ (0.0831*)\end{tabular} \\ \cline{2-7}
 & uUQ & \begin{tabular}[c]{@{}c@{}}0.1335\\ (0.1131)\end{tabular} & \begin{tabular}[c]{@{}c@{}}0.1110\\ (0.1128)\end{tabular} & \begin{tabular}[c]{@{}c@{}}0.1061\\ (0.1087*)\end{tabular} & \begin{tabular}[c]{@{}c@{}}0.1362\\ (0.1038*)\end{tabular} & \begin{tabular}[c]{@{}c@{}}0.1180\\ (0.0989*)\end{tabular} \\ \cline{2-7}
\multirow{-4}{*}{32} & uECSQ & \begin{tabular}[c]{@{}c@{}}0.8706\\ (0.1172)\end{tabular} & \begin{tabular}[c]{@{}c@{}}0.8731\\ (0.1176)\end{tabular} & \begin{tabular}[c]{@{}c@{}}0.8731\\ (0.1139)\end{tabular} & \begin{tabular}[c]{@{}c@{}}0.8668\\ (0.1049*)\end{tabular} & \begin{tabular}[c]{@{}c@{}}0.8794\\ (0.0992*)\end{tabular} \\ \hline
 & uCWS & \begin{tabular}[c]{@{}c@{}}0.8898\\ (0.1341)\end{tabular} & \begin{tabular}[c]{@{}c@{}}0.8904\\ (0.1328)\end{tabular} & \begin{tabular}[c]{@{}c@{}}0.8892\\ (0.1274*)\end{tabular} & \begin{tabular}[c]{@{}c@{}}0.8913\\ (0.1213*)\end{tabular} & \begin{tabular}[c]{@{}c@{}}0.8923\\ (0.1154*)\end{tabular} \\ \cline{2-7}
 & uPWS & \begin{tabular}[c]{@{}c@{}}0.1210\\ (0.1165)\end{tabular} & \begin{tabular}[c]{@{}c@{}}0.1216\\ (0.1153)\end{tabular} & \begin{tabular}[c]{@{}c@{}}0.1158\\ (0.1105*)\end{tabular} & \begin{tabular}[c]{@{}c@{}}0.1156\\ (0.1048*)\end{tabular} & \begin{tabular}[c]{@{}c@{}}0.1098\\ (0.0992*)\end{tabular} \\ \cline{2-7}
 & uUQ & \begin{tabular}[c]{@{}c@{}}0.1385\\ (0.1292)\end{tabular} & \begin{tabular}[c]{@{}c@{}}0.1601\\ (0.1289)\end{tabular} & \begin{tabular}[c]{@{}c@{}}0.1659\\ (0.1255*)\end{tabular} & \begin{tabular}[c]{@{}c@{}}0.1673\\ (0.1203*)\end{tabular} & \begin{tabular}[c]{@{}c@{}}0.1678\\ (0.1152*)\end{tabular} \\ \cline{2-7}
\multirow{-4}{*}{64} & uECSQ & \begin{tabular}[c]{@{}c@{}}0.8822\\ (0.1348)\end{tabular} & \begin{tabular}[c]{@{}c@{}}0.8774\\ (0.1331)\end{tabular} & \begin{tabular}[c]{@{}c@{}}0.8799\\ (0.1307)\end{tabular} & \begin{tabular}[c]{@{}c@{}}0.8799\\ (0.1214*)\end{tabular} & \begin{tabular}[c]{@{}c@{}}0.8757\\ (0.1154*)\end{tabular} \\ \hline
 & uCWS & \begin{tabular}[c]{@{}c@{}}0.8919\\ (0.1514)\end{tabular} & \begin{tabular}[c]{@{}c@{}}0.8944\\ (0.1497)\end{tabular} & \begin{tabular}[c]{@{}c@{}}0.8936\\ (0.1441*)\end{tabular} & \begin{tabular}[c]{@{}c@{}}0.9022\\ (0.1376*)\end{tabular} & \begin{tabular}[c]{@{}c@{}}0.9027\\ (0.1315*)\end{tabular} \\ \cline{2-7}
 & uPWS & \begin{tabular}[c]{@{}c@{}}0.1210\\ (0.1339)\end{tabular} & \begin{tabular}[c]{@{}c@{}}0.1164\\ (0.1324)\end{tabular} & \begin{tabular}[c]{@{}c@{}}0.1290\\ (0.1272*)\end{tabular} & \begin{tabular}[c]{@{}c@{}}0.1172\\ (0.1212*)\end{tabular} & \begin{tabular}[c]{@{}c@{}}0.1263\\ (0.1154*)\end{tabular} \\ \cline{2-7}
 & uUQ & \begin{tabular}[c]{@{}c@{}}0.1440\\ (0.1297)\end{tabular} & \begin{tabular}[c]{@{}c@{}}0.1409\\ (0.1294)\end{tabular} & \begin{tabular}[c]{@{}c@{}}0.1414\\ (0.1258*)\end{tabular} & \begin{tabular}[c]{@{}c@{}}0.1494\\ (0.1206*)\end{tabular} & \begin{tabular}[c]{@{}c@{}}0.1323\\ (0.1152*)\end{tabular} \\ \cline{2-7}
\multirow{-4}{*}{128} & uECSQ & \begin{tabular}[c]{@{}c@{}}0.8917\\ (0.1523)\end{tabular} & \begin{tabular}[c]{@{}c@{}}0.8889\\ (0.1503)\end{tabular} & \begin{tabular}[c]{@{}c@{}}0.8899\\ (0.1475)\end{tabular} & \begin{tabular}[c]{@{}c@{}}0.8955\\ (0.1379*)\end{tabular} & \begin{tabular}[c]{@{}c@{}}0.8995\\ (0.1316*)\end{tabular} \\ \hline
 & uCWS & \begin{tabular}[c]{@{}c@{}}0.9083\\ (0.1693)\end{tabular} & \begin{tabular}[c]{@{}c@{}}0.9089\\ (0.1672)\end{tabular} & \begin{tabular}[c]{@{}c@{}}0.9089\\ (0.1611*)\end{tabular} & \begin{tabular}[c]{@{}c@{}}0.9088\\ (0.1544*)\end{tabular} & \begin{tabular}[c]{@{}c@{}}0.9080\\ (0.1478*)\end{tabular} \\ \cline{2-7}
 & uPWS & \begin{tabular}[c]{@{}c@{}}0.1270\\ (0.1515)\end{tabular} & \begin{tabular}[c]{@{}c@{}}0.1239\\ (0.1497)\end{tabular} & \begin{tabular}[c]{@{}c@{}}0.1207\\ (0.1440*)\end{tabular} & \begin{tabular}[c]{@{}c@{}}0.1244\\ (0.1377*)\end{tabular} & \begin{tabular}[c]{@{}c@{}}0.1335\\ (0.1315*)\end{tabular} \\ \cline{2-7}
 & uUQ & \begin{tabular}[c]{@{}c@{}}0.9317\\ (0.1461)\end{tabular} & \begin{tabular}[c]{@{}c@{}}0.9305\\ (0.1457)\end{tabular} & \begin{tabular}[c]{@{}c@{}}0.9331\\ (0.1420*)\end{tabular} & \begin{tabular}[c]{@{}c@{}}0.9328\\ (0.1367*)\end{tabular} & \begin{tabular}[c]{@{}c@{}}0.9309\\ (0.1314*)\end{tabular} \\ \cline{2-7}
\multirow{-4}{*}{256} & uECSQ & \begin{tabular}[c]{@{}c@{}}0.9034\\ (0.1699)\end{tabular} & \begin{tabular}[c]{@{}c@{}}0.9040\\ (0.1676)\end{tabular} & \begin{tabular}[c]{@{}c@{}}0.9031\\ (0.1643)\end{tabular} & \begin{tabular}[c]{@{}c@{}}0.9052\\ (0.1544*)\end{tabular} & \begin{tabular}[c]{@{}c@{}}0.9037\\ (0.1478*)\end{tabular} \\ \hline
\end{tabular}
\end{table*}

\begin{table*}[h]
\centering
\caption{Summary of the results of the application of pruning and quantization applied to Dense layers and quantization on Convolutional layers of DeepDTA trained on the Kiba dataset. Performance is reported as MSE; compression ratio ($\psi$) in brackets. Baseline of the uncompressed model reported in bold.}
\label{tab:pr-and-q-dense-and-cnn-kiba}
\begin{tabular}{|c|l|ccccc|}
\hline
\rowcolor[HTML]{B4A7D6}
\multicolumn{7}{|c|}{\cellcolor[HTML]{B4A7D6}\textbf{DeepDTA - KIBA (0.1756)}} \\ \hline
\rowcolor[HTML]{9FC5E8}
\cellcolor[HTML]{B6D7A8} & \multicolumn{1}{c|}{\cellcolor[HTML]{B6D7A8}} & \multicolumn{5}{c|}{\cellcolor[HTML]{9FC5E8}\textbf{ }} \\ \cline{3-7}
\rowcolor[HTML]{9FC5E8}
\multirow{-2}{*}{\cellcolor[HTML]{B6D7A8}\textbf{$k$ (CNN and Dense)}} & \multicolumn{1}{c|}{\multirow{-2}{*}{\cellcolor[HTML]{B6D7A8}\textbf{Method}}} & \multicolumn{1}{c|}{\cellcolor[HTML]{9FC5E8}\textbf{PR Dense: 50}} & \multicolumn{1}{c|}{\cellcolor[HTML]{9FC5E8}\textbf{PR Dense: 55}} & \multicolumn{1}{c|}{\cellcolor[HTML]{9FC5E8}\textbf{PR Dense: 60}} & \multicolumn{1}{c|}{\cellcolor[HTML]{9FC5E8}\textbf{PR Dense: 65}} & \textbf{PR Dense: 70} \\ \hline
 & uCWS & \begin{tabular}[c]{@{}c@{}}0.1680\\ (0.1149)\end{tabular} & \begin{tabular}[c]{@{}c@{}}0.1693\\ (0.1085)\end{tabular} & \begin{tabular}[c]{@{}c@{}}0.1717\\ (0.1015)\end{tabular} & \begin{tabular}[c]{@{}c@{}}0.1780\\ (0.0953)\end{tabular} & \begin{tabular}[c]{@{}c@{}}0.1954\\ (0.0884)\end{tabular} \\ \cline{2-2}
 & uPWS & \begin{tabular}[c]{@{}c@{}}0.1918\\ (0.1021)\end{tabular} & \begin{tabular}[c]{@{}c@{}}0.1978\\ (0.0950)\end{tabular} & \begin{tabular}[c]{@{}c@{}}0.1908\\ (0.0883)\end{tabular} & \begin{tabular}[c]{@{}c@{}}0.3108\\ (0.0779)\end{tabular} & \begin{tabular}[c]{@{}c@{}}0.3946\\ (0.0712)\end{tabular} \\ \cline{2-2}
 & uUQ & \begin{tabular}[c]{@{}c@{}}0.1686\\ (0.0869)\end{tabular} & \begin{tabular}[c]{@{}c@{}}0.1690\\ (0.0839)\end{tabular} & \begin{tabular}[c]{@{}c@{}}0.1732\\ (0.0806)\end{tabular} & \begin{tabular}[c]{@{}c@{}}0.1987\\ (0.0773)\end{tabular} & \begin{tabular}[c]{@{}c@{}}0.3272\\ (0.0733)\end{tabular} \\ \cline{2-2}
\multirow{-4}{*}{32} & uECSQ & \begin{tabular}[c]{@{}c@{}}0.1679\\ (0.1272)\end{tabular} & \begin{tabular}[c]{@{}c@{}}0.1684\\ (0.1189)\end{tabular} & \begin{tabular}[c]{@{}c@{}}0.1748\\ (0.1094)\end{tabular} & \begin{tabular}[c]{@{}c@{}}0.1777\\ (0.1002)\end{tabular} & \begin{tabular}[c]{@{}c@{}}0.1970\\ (0.0913)\end{tabular} \\ \hline
 & uCWS & \begin{tabular}[c]{@{}c@{}}0.1649\\ (0.1320)\end{tabular} & \begin{tabular}[c]{@{}c@{}}0.1664\\ (0.1240)\end{tabular} & \begin{tabular}[c]{@{}c@{}}0.1704\\ (0.1158)\end{tabular} & \begin{tabular}[c]{@{}c@{}}0.1758\\ (0.1088)\end{tabular} & \begin{tabular}[c]{@{}c@{}}0.1938\\ (0.1007)\end{tabular} \\ \cline{2-2}
 & uPWS & \begin{tabular}[c]{@{}c@{}}0.1781\\ (0.1198)\end{tabular} & \begin{tabular}[c]{@{}c@{}}0.1789\\ (0.1115)\end{tabular} & \begin{tabular}[c]{@{}c@{}}0.1827\\ (0.1032)\end{tabular} & \begin{tabular}[c]{@{}c@{}}0.1956\\ (0.0913)\end{tabular} & \begin{tabular}[c]{@{}c@{}}0.3533\\ (0.0833)\end{tabular} \\ \cline{2-2}
 & uUQ & \begin{tabular}[c]{@{}c@{}}0.1658\\ (0.1040)\end{tabular} & \begin{tabular}[c]{@{}c@{}}0.1674\\ (0.0987)\end{tabular} & \begin{tabular}[c]{@{}c@{}}0.1742\\ (0.9418)\end{tabular} & \begin{tabular}[c]{@{}c@{}}0.1906\\ (0.0897)\end{tabular} & \begin{tabular}[c]{@{}c@{}}0.1985\\ (0.0848)\end{tabular} \\ \cline{2-2}
\multirow{-4}{*}{64} & uECSQ & \begin{tabular}[c]{@{}c@{}}0.1656\\ (0.1423)\end{tabular} & \begin{tabular}[c]{@{}c@{}}0.1665\\ (0.1326)\end{tabular} & \begin{tabular}[c]{@{}c@{}}0.1700\\ (0.1234)\end{tabular} & \begin{tabular}[c]{@{}c@{}}0.1785\\ (0.1128)\end{tabular} & \begin{tabular}[c]{@{}c@{}}0.4559\\ (0.1034)\end{tabular} \\ \hline
 & uCWS & \begin{tabular}[c]{@{}c@{}}0.1657\\ (0.1514)\end{tabular} & \begin{tabular}[c]{@{}c@{}}0.1670\\ (0.1419)\end{tabular} & \begin{tabular}[c]{@{}c@{}}0.1702\\ (0.1323)\end{tabular} & \begin{tabular}[c]{@{}c@{}}0.1769\\ (0.1229)\end{tabular} & \begin{tabular}[c]{@{}c@{}}0.4429\\ (0.1141)\end{tabular} \\ \cline{2-2}
 & uPWS & \begin{tabular}[c]{@{}c@{}}0.1734\\ (0.1377)\end{tabular} & \begin{tabular}[c]{@{}c@{}}0.1741\\ (0.1278)\end{tabular} & \begin{tabular}[c]{@{}c@{}}0.1735\\ (0.1144)\end{tabular} & \begin{tabular}[c]{@{}c@{}}0.1876\\ (0.1050)\end{tabular} & \begin{tabular}[c]{@{}c@{}}0.2804\\ (0.0956)\end{tabular} \\ \cline{2-2}
 & uUQ & \begin{tabular}[c]{@{}c@{}}0.1652\\ (0.1218)\end{tabular} & \begin{tabular}[c]{@{}c@{}}0.1670\\ (0.1162)\end{tabular} & \begin{tabular}[c]{@{}c@{}}0.1734\\ (0.1097)\end{tabular} & \begin{tabular}[c]{@{}c@{}}0.1866\\ (0.1036)\end{tabular} & \begin{tabular}[c]{@{}c@{}}0.1946\\ (0.0974)\end{tabular} \\ \cline{2-2}
\multirow{-4}{*}{128} & uECSQ & \begin{tabular}[c]{@{}c@{}}0.1649\\ (0.1546)\end{tabular} & \begin{tabular}[c]{@{}c@{}}0.1678\\ (0.1463)\end{tabular} & \begin{tabular}[c]{@{}c@{}}0.1707\\ (0.1353)\end{tabular} & \begin{tabular}[c]{@{}c@{}}0.4311\\ (0.1257)\end{tabular} & \begin{tabular}[c]{@{}c@{}}0.4977\\ (0.1148)\end{tabular} \\ \hline
 & uCWS & \begin{tabular}[c]{@{}c@{}}0.1652\\ (0.1690)\end{tabular} & \begin{tabular}[c]{@{}c@{}}0.1669\\ (0.1585)\end{tabular} & \begin{tabular}[c]{@{}c@{}}0.1705\\ (0.1480)\end{tabular} & \begin{tabular}[c]{@{}c@{}}0.4485\\ (0.1371)\end{tabular} & \begin{tabular}[c]{@{}c@{}}0.4111\\ (0.1264)\end{tabular} \\ \cline{2-2}
 & uPWS & \begin{tabular}[c]{@{}c@{}}0.1720\\ (0.1555)\end{tabular} & \begin{tabular}[c]{@{}c@{}}0.1737\\ (0.1441)\end{tabular} & \begin{tabular}[c]{@{}c@{}}0.1762\\ (0.1294)\end{tabular} & \begin{tabular}[c]{@{}c@{}}0.3684\\ (0.1186)\end{tabular} & \begin{tabular}[c]{@{}c@{}}0.4418\\ (0.1078)\end{tabular} \\ \cline{2-2}
 & uUQ & \begin{tabular}[c]{@{}c@{}}0.1654\\ (0.1409)\end{tabular} & \begin{tabular}[c]{@{}c@{}}0.1662\\ (0.1333)\end{tabular} & \begin{tabular}[c]{@{}c@{}}0.1706\\ (0.1259)\end{tabular} & \begin{tabular}[c]{@{}c@{}}0.1766\\ (0.1183)\end{tabular} & \begin{tabular}[c]{@{}c@{}}0.1947\\ (0.1103)\end{tabular} \\ \cline{2-2}
\multirow{-4}{*}{256} & uECSQ & \begin{tabular}[c]{@{}c@{}}0.1659\\ (0.1651)\end{tabular} & \begin{tabular}[c]{@{}c@{}}0.1673\\ (0.1557)\end{tabular} & \begin{tabular}[c]{@{}c@{}}0.1714\\ (0.1453)\end{tabular} & \begin{tabular}[c]{@{}c@{}}0.4221\\ (0.1312)\end{tabular} & \begin{tabular}[c]{@{}c@{}}0.5349\\ (0.1202)\end{tabular} \\ \hline
\end{tabular}
\end{table*}

\begin{table*}[h]
\centering
\caption{Summary of the results of the application of pruning and quantization applied to Dense layers and quantization on Convolutional layers of DeepDTA trained on the Davis dataset. Performance is reported as MSE; compression ratio ($\psi$) in brackets. Baseline of the uncompressed model reported in bold.}
\label{tab:pr-and-q-dense-and-cnn-davis}
\begin{tabular}{|c|l|ccccc|}
\hline
\rowcolor[HTML]{B4A7D6}
\multicolumn{7}{|c|}{\cellcolor[HTML]{B4A7D6}\textbf{DeepDTA - DAVIS (0.3223)}} \\ \hline
\rowcolor[HTML]{9FC5E8}
\cellcolor[HTML]{B6D7A8} & \multicolumn{1}{c|}{\cellcolor[HTML]{B6D7A8}} & \multicolumn{5}{c|}{\cellcolor[HTML]{9FC5E8}\textbf{ }} \\ \cline{3-7}
\rowcolor[HTML]{9FC5E8}
\multirow{-2}{*}{\cellcolor[HTML]{B6D7A8}\textbf{$k$ (CNN and Dense)}} & \multicolumn{1}{c|}{\multirow{-2}{*}{\cellcolor[HTML]{B6D7A8}\textbf{Method}}} & \multicolumn{1}{c|}{\cellcolor[HTML]{9FC5E8}\textbf{PR Dense: 70}} & \multicolumn{1}{c|}{\cellcolor[HTML]{9FC5E8}\textbf{PR Dense: 75}} & \multicolumn{1}{c|}{\cellcolor[HTML]{9FC5E8}\textbf{PR Dense: 80}} & \multicolumn{1}{c|}{\cellcolor[HTML]{9FC5E8}\textbf{PR Dense: 85}} & \textbf{PR Dense: 90} \\ \hline
 & uCWS & \begin{tabular}[c]{@{}c@{}}0.2309\\ (0.0866)\end{tabular} & \begin{tabular}[c]{@{}c@{}}0.2342\\ (0.0789)\end{tabular} & \begin{tabular}[c]{@{}c@{}}0.2408\\ (0.0733)\end{tabular} & \begin{tabular}[c]{@{}c@{}}0.2544\\ (0.0666)\end{tabular} & \begin{tabular}[c]{@{}c@{}}0.2530\\ (0.0607)\end{tabular} \\ \cline{2-2}
 & uPWS & \begin{tabular}[c]{@{}c@{}}0.4050\\ (0.0702)\end{tabular} & \begin{tabular}[c]{@{}c@{}}0.3990\\ (0.0642)\end{tabular} & \begin{tabular}[c]{@{}c@{}}0.4180\\ (0.0584)\end{tabular} & \begin{tabular}[c]{@{}c@{}}0.3626\\ (0.0528)\end{tabular} & \begin{tabular}[c]{@{}c@{}}0.5045\\ (0.0486)\end{tabular} \\ \cline{2-2}
 & uUQ & \begin{tabular}[c]{@{}c@{}}0.2331\\ (0.0644)\end{tabular} & \begin{tabular}[c]{@{}c@{}}0.2376\\ (0.0633)\end{tabular} & \begin{tabular}[c]{@{}c@{}}0.2397\\ (0.0603)\end{tabular} & \begin{tabular}[c]{@{}c@{}}0.2482\\ (0.0575)\end{tabular} & \begin{tabular}[c]{@{}c@{}}0.2801\\ (0.0544)\end{tabular} \\ \cline{2-2}
\multirow{-4}{*}{32} & uECSQ & \begin{tabular}[c]{@{}c@{}}0.2310\\ (0.0901)\end{tabular} & \begin{tabular}[c]{@{}c@{}}0.2402\\ (0.0771)\end{tabular} & \begin{tabular}[c]{@{}c@{}}0.2441\\ (0.0712)\end{tabular} & \begin{tabular}[c]{@{}c@{}}0.2441\\ (0.0634)\end{tabular} & \begin{tabular}[c]{@{}c@{}}0.2552\\ (0.0564)\end{tabular} \\ \hline
 & uCWS & \begin{tabular}[c]{@{}c@{}}0.2386\\ (0.0995)\end{tabular} & \begin{tabular}[c]{@{}c@{}}0.2346\\ (0.0920)\end{tabular} & \begin{tabular}[c]{@{}c@{}}0.2372\\ (0.0844)\end{tabular} & \begin{tabular}[c]{@{}c@{}}0.2524\\ (0.0663)\end{tabular} & \begin{tabular}[c]{@{}c@{}}0.2603\\ (0.0590)\end{tabular} \\ \cline{2-2}
 & uPWS & \begin{tabular}[c]{@{}c@{}}0.2539\\ (0.0825)\end{tabular} & \begin{tabular}[c]{@{}c@{}}0.2552\\ (0.0747)\end{tabular} & \begin{tabular}[c]{@{}c@{}}0.2610\\ (0.0679)\end{tabular} & \begin{tabular}[c]{@{}c@{}}0.2842\\ (0.061)\end{tabular} & \begin{tabular}[c]{@{}c@{}}0.3271\\ (0.0549)\end{tabular} \\ \cline{2-2}
 & uUQ & \begin{tabular}[c]{@{}c@{}}0.2344\\ (0.0748)\end{tabular} & \begin{tabular}[c]{@{}c@{}}0.2369\\ (0.0719)\end{tabular} & \begin{tabular}[c]{@{}c@{}}0.2410\\ (0.0685)\end{tabular} & \begin{tabular}[c]{@{}c@{}}0.2525\\ (0.0645)\end{tabular} & \begin{tabular}[c]{@{}c@{}}0.2584\\ (0.0597)\end{tabular} \\ \cline{2-2}
\multirow{-4}{*}{64} & uECSQ & \begin{tabular}[c]{@{}c@{}}0.2364\\ (0.0949)\end{tabular} & \begin{tabular}[c]{@{}c@{}}0.2348\\ (0.0866)\end{tabular} & \begin{tabular}[c]{@{}c@{}}0.2356\\ (0.0805)\end{tabular} & \begin{tabular}[c]{@{}c@{}}0.2414\\ (0.0714)\end{tabular} & \begin{tabular}[c]{@{}c@{}}0.2467\\ (0.0632)\end{tabular} \\ \hline
 & uCWS & \begin{tabular}[c]{@{}c@{}}0.2311\\ (0.1114)\end{tabular} & \begin{tabular}[c]{@{}c@{}}0.2301\\ (0.1026)\end{tabular} & \begin{tabular}[c]{@{}c@{}}0.2440\\ (0.0791)\end{tabular} & \begin{tabular}[c]{@{}c@{}}0.2412\\ (0.0705)\end{tabular} & \begin{tabular}[c]{@{}c@{}}0.2525\\ (0.0598)\end{tabular} \\ \cline{2-2}
 & uPWS & \begin{tabular}[c]{@{}c@{}}0.2535\\ (0.0949)\end{tabular} & \begin{tabular}[c]{@{}c@{}}0.2474\\ (0.0857)\end{tabular} & \begin{tabular}[c]{@{}c@{}}0.2538\\ (0.0772)\end{tabular} & \begin{tabular}[c]{@{}c@{}}0.2637\\ (0.0690)\end{tabular} & \begin{tabular}[c]{@{}c@{}}0.2755\\ (0.0576)\end{tabular} \\ \cline{2-2}
 & uUQ & \begin{tabular}[c]{@{}c@{}}0.2295\\ (0.0860)\end{tabular} & \begin{tabular}[c]{@{}c@{}}0.2305\\ (0.0818)\end{tabular} & \begin{tabular}[c]{@{}c@{}}0.2342\\ (0.0763)\end{tabular} & \begin{tabular}[c]{@{}c@{}}0.2423\\ (0.0716)\end{tabular} & \begin{tabular}[c]{@{}c@{}}0.2636\\ (0.0667)\end{tabular} \\ \cline{2-2}
\multirow{-4}{*}{128} & uECSQ & \begin{tabular}[c]{@{}c@{}}0.2343\\ (0.1031)\end{tabular} & \begin{tabular}[c]{@{}c@{}}0.2335\\ (0.0934)\end{tabular} & \begin{tabular}[c]{@{}c@{}}0.2355\\ (0.0898)\end{tabular} & \begin{tabular}[c]{@{}c@{}}0.2377\\ (0.0794)\end{tabular} & \begin{tabular}[c]{@{}c@{}}0.2570\\ (0.0698)\end{tabular} \\ \hline
 & uCWS & \begin{tabular}[c]{@{}c@{}}0.2267\\ (0.1252)\end{tabular} & \begin{tabular}[c]{@{}c@{}}0.2273\\ (0.1145)\end{tabular} & \begin{tabular}[c]{@{}c@{}}0.2346\\ (0.0862)\end{tabular} & \begin{tabular}[c]{@{}c@{}}0.2411\\ (0.0770)\end{tabular} & \begin{tabular}[c]{@{}c@{}}0.2718\\ (0.0644)\end{tabular} \\ \cline{2-2}
 & uPWS & \begin{tabular}[c]{@{}c@{}}0.2397\\ (0.1071)\end{tabular} & \begin{tabular}[c]{@{}c@{}}0.2390\\ (0.0965)\end{tabular} & \begin{tabular}[c]{@{}c@{}}0.2440\\ (0.0865)\end{tabular} & \begin{tabular}[c]{@{}c@{}}0.2460\\ (0.0770)\end{tabular} & \begin{tabular}[c]{@{}c@{}}0.2581\\ (0.0642)\end{tabular} \\ \cline{2-2}
 & uUQ & \begin{tabular}[c]{@{}c@{}}0.2305\\ (0.0943)\end{tabular} & \begin{tabular}[c]{@{}c@{}}0.2391\\ (0.0882)\end{tabular} & \begin{tabular}[c]{@{}c@{}}0.2393\\ (0.0822)\end{tabular} & \begin{tabular}[c]{@{}c@{}}0.2419\\ (0.0760)\end{tabular} & \begin{tabular}[c]{@{}c@{}}0.2505\\ (0.0698)\end{tabular} \\ \cline{2-2}
\multirow{-4}{*}{256} & uECSQ & \begin{tabular}[c]{@{}c@{}}0.2291\\ (0.1103)\end{tabular} & \begin{tabular}[c]{@{}c@{}}0.2286\\ (0.1020)\end{tabular} & \begin{tabular}[c]{@{}c@{}}0.2329\\ (0.0995)\end{tabular} & \begin{tabular}[c]{@{}c@{}}0.2419\\ (0.0876)\end{tabular} & \begin{tabular}[c]{@{}c@{}}0.2663\\ (0.0765)\end{tabular} \\ \hline
\end{tabular}
\end{table*}